\documentclass[10pt]{amsart}
\usepackage[margin=3.5cm]{geometry}

\usepackage{amsmath,amssymb,mathrsfs,latexsym,color,txfonts,array,graphicx}
\usepackage[all]{xy}
\usepackage{multicol}
\usepackage[ruled,vlined,norelsize]{algorithm2e}
\usepackage{float}
\usepackage{wrapfig}

\author{Achilles A. Beros}
\author{Colin de la Higuera}

\title{A Canonical Semi-Deterministic Transducer}


\newtheorem{Theorem}{Theorem}[section]
\newtheorem{Proposition}[Theorem]{Proposition}
\newtheorem{Lemma}[Theorem]{Lemma}
\newtheorem{Corollary}[Theorem]{Corollary}

\theoremstyle{definition}
\newtheorem{Definition}[Theorem]{Definition}

\newtheorem{Example}[Theorem]{Example}

\newcommand{\vac} {\mathsf{Vac}}

\newcommand{\states}{\mbox{\sc{states}}}
\newcommand{\future}{\mbox{\sc{future}}}

\newcommand{\upto} {{\upharpoonright}}

\newcommand{\lexorder}{LEXORDER}
\newcommand{\llexorder}{LLEXORDER}
\newcommand{\lexleast}{LEXLEAST}
\newcommand{\llexleast}{LLEXLEAST}
\newcommand{\lllesdf}[1]{\overline{#1}}
\newcommand{\trinput}{input}
\newcommand{\troutput}{output}

\begin{document}


\maketitle

\begin{abstract}
We prove the existence of a canonical form for semi-deterministic transducers with sets of pairwise incomparable output strings.  Based on this, we develop an algorithm which learns semi-deterministic transducers given access to translation queries.  We also prove that there is no learning algorithm for semi-deterministic transducers that uses only domain knowledge.
\end{abstract}

\section{Introduction}

Transducers, introduced by \cite{nivat68}, are a type of abstract machine which defines a relation between two formal languages.  As such, they are interpreted as modeling translation in any context where formal languages are applicable. We provide no background on formal languages in this paper; an overview of the subject can be found in \cite{bers79} and \cite{saka09}.  Alternatively, transducers can be viewed as a generalization of finite state machines.  This view was introduced by Mohri, who uses transducers in the context of natural language processing \cite{mohr97,mohr00a} and \cite{mohr00b}.

A fundamental task when studying the theory of transducers is to look for classes of transducers that can be learned given access to some form of data.  If a class of transducers, $\mathscr C$, is found to be learnable, then a predictive model can be produced in any application where a translation from the class $\mathscr C$ is in use.  The significance of transducers, specifically expanding the range of the learnable classes, is clear from the scope of applications of transducers.  Among many others, some well known applications are in the fields of morphology and phonology \cite{roar07}, machine translation \cite{amen01a,casa04,clar01b}, web wrappers \cite{carm05}, speech \cite{mohr97} and pattern recognition \cite{bern06}.  In each of these cases, different classes of transducers are examined with characteristics suitable to the application.  Distinguishing characteristics of different classes include determinism properties, the use of probabilites or weights, as well as details of the types of transitions that are permitted.

\subsection{Transducer learning}

An important step in the theory of transducers was the development of the algorithm \textsc{Ostia}.  Introduced in \cite{onci93}, \textsc{Ostia} was designed for language comprehension tasks \cite{cast93}.  A number of elaborations on the original algorithm have since arisen, many of them aimed at trying to circumvent the restriction to total functions that limited \textsc{Ostia}.  Typically, these attempts involved adding some new source of information.  For example, \textsc{Ostia}-N uses negative (input) examples and \textsc{Ostia}-D supposes the algorithm has some knowledge of the domain of the function \cite{onci96}.  Similar ideas were explored later by \cite{kerm02a} and \cite{cost04}.  An application of \textsc{Ostia} for active learning is presented in \cite{vila96}. Using dictionaries and word alignments has been tested by \cite{vila00}.  A demonstrated practical success of \textsc{Ostia} came in 2006.  The Tenjinno competition \cite{star06} was won by \cite{clar06a} using an \textsc{Ostia} inspired algorithm.

\subsection{Towards nondeterminism with transducers}

Non-deterministic transducers pose numerous complex questions -- even parsing becomes a difficult problem \cite{casa99,casa00a}.  Interest in non-deterministic models remains, however, as the limitations of subsequential transducers make them unacceptable for most applications.  The first lifting of these constraints was proposed by \cite{alla02}.  They propose a model in which the final states may have multiple outputs.  In his PhD thesis, Akram introduced a notion of semi-determinism \cite{akra13} that strikes a balance between complete non-determinism and the very restrictive subsequential class.  He provided an example witnessing that semi-deterministic transducers are a proper generalization of deterministic transducers, but did not pursue the topic further, focusing instead on probabilistic subsequential transducers.  We examine an equivalent formulation of Akram's semi-determinism based on methods of mathematical logic.  In particular, by viewing the definition from a higher level of the ranked universe, we convert what would be a general relation into a well-defined function.  \cite{kunen80} provides an overview of a number of important topics in set theory including the ranked and definable universes.  Some more recent developments in set theory is \cite{jech2003set}.

A significant obstacle in learning non-deterministic transducers is the fact that an absence of information cannot be interpreted.  One approach to overcoming this problem is to use probabilities.  We eschew the probabilistic approach in favor of a collection of methods that have their antecedents in Beros's earlier work distinguishing learning models \cite{beros2013} and determining the arithmetic complexity of learning models \cite{berosND}.

An earlier version of this work was presented at the International Conference on Grammatical Inference \cite{beros-delahiguera-2014}.  In this version, we provide more of the algorithms involved in learning semi-deterministic transducers and prove that the algorithms converge.  We also establish the relationship between semi-deterministic transducers and two other natural extensions of deterministic transducers and the bi-languages they generate, specifically $p$-subsequential transducers and finitary, finite-state, and bounded relations (definitions of these terms are provided in Section \ref{other-non-det-section}).  Finally, we show that semi-deterministic transducers and the associated bi-languages fail two closure properties: closure under composition and closure under bi-language reversal.

\section{Notation}

We make use of following common notation in the course of this paper.  Throughout, the symbols $x,y$ and $z$ denote strings and $a$ and $b$ will denote elements of a given alphabet.  We shall use the standard notation $\lambda$ for the empty string.

\begin{itemize}
\item The concatenation of two strings, $x$ and $y$, is denoted by $xy$.  We write $x \prec y$ if there is a string $z \neq \lambda$ such that $y = xz$.  We write $x \preceq y$ if $x \prec y$ or $x = y$.  This order is called the prefix order.

\item For a set of strings, $S$, $T[S] = \{ x : (\exists y\in S)\big( x \preceq y \big) \}$ is the prefix closure of $S$.

\item A tree is a set of strings, $S$, such that $T[S] = S$.  $S'$ is a subtree of $S$ if both $S$ and $S'$ are trees and $S'$ is contained in $S$.  A strict subtree is a subtree that is not equal to the containing tree.  

\item $\mathscr P(X) = \{ Y: Y\subseteq X \}$ and $\mathscr P^*(X) = \{ Y: Y\subseteq X \wedge |Y|<\infty \}$.

\item We will use elements of $\mathbb N$ both as numbers and as sets.  In particular, we use the following inductive definition: $0 = \emptyset$ and, given $0, \ldots , n$, we define $n+1 = \{0, \ldots , n\}$.

\item Following the notation of set theory, the string $x = a_0 \ldots a_n$ is a function with domain $n+1$.  Thus, $x \upto k = a_0 \ldots a_{k-1}$ for $k \leq n+1$.  $|x|$ is the length of $x$ and $x^-$ is the truncation $x \upto (|x| - 1)$.  Note that the last element of $x$ is $x(|x|-1)$ and the last element of $x^-$ is $x(|x|-2)$. 

\item Again, drawing on set theory terminology, we call two functions, $f$ and $g$, compatible if $(\forall x \in \mbox{dom}(f) \cap \mbox{dom}(g))(f(x) = g(x))$.

\item We write $x\parallel y$ if $x = y$, $x \prec y$ or $x \succ y$ and say $x$ and $y$ are comparable.  Otherwise, we write $x \perp y$ and say that $x$ and $y$ are incomparable.

\item By $<_{lex}$ and $<_{llex}$ we denote the lexicographic and length-lexicographic orders, respectively.

\item For an alphabet $\Sigma$, $\Sigma^*$ is the set of all finite strings over $\Sigma$.  A tree over $\Sigma$ is a tree whose members are members of $\Sigma^*$, where the ordering of the tree is consistent with the prefix order on $\Sigma^*$ and the tree is prefix closed. 

\item We reserve a distinguished character, \#, which we exclude from all alphabets under consideration and we will use \# to indicate the end of a word.  We will write $x \#$ when we append the \# character to $x$.
\end{itemize}


\section{Bi-Languages and Transducers}\label{bilang-trans-section}

Bi-languages are the fundamental objects of study.  They capture the semantic correspondence between two languages.  In principle, this correspondence does not specify any ordering of the two languages, but translation is always done from one language \emph{to} another language.  As such, we refer to the input and the output languages of a bi-language.  For notational simplicity, in everything that follows $\Sigma$ is the alphabet for input languages and $\Omega$ is the alphabet for output languages.  Using this notation, the input language is a subset of $\Sigma^*$ and the output language is a subset of $\Omega^*$.  We now present the standard definition of a bi-language.

\begin{Definition}
Consider two languages, $L \subseteq \Sigma^*$ and $K \subseteq \Omega^*$.  A \emph{bi-language from $L$ to $K$} is a subset of $L \times K$ with domain $L$.
\end{Definition}

For our purposes, we wish to indicate the direction of translation and to aggregate all translations of a single string.  To this end, in the remainder of this paper, we will use the following equivalent definition of a bi-language.

\begin{Definition}\label{bi-L}
Consider two languages, $L \subseteq \Sigma^*$ and $K \subseteq \Omega^*$.  A \emph{bi-language from $L$ to $K$} is a function $f: L \rightarrow \mathscr P(K)$.  $L$ is said to be the \emph{input language} and $K$ the \emph{output language} of $f$.  When defined without reference to a specific output language, a bi-language is simply a function $f:L \rightarrow \mathscr P(\Omega^*)$.  If $f$ and $g$ are two bi-languages, then $f$ is a \emph{sub bi-language} of $g$ if $\mbox{dom}(f) \subseteq \mbox{dom}(g)$ and for all $x \in \mbox{dom}(f)$, $f(x) \subseteq g(x)$.  A finite subset $\mathcal D$ of $L \times K$ is \emph{consistent with $f$} if for every $\langle x,X \rangle \in \mathcal D$, $X \in f(x)$.
\end{Definition}

Note that for a bi-language $f$ from $L$ to $K$, we do not require that $\bigcup_{x \in L} f(x) = K$.  We are interested in languages whose generating syntax is some form of transducer.


\begin{Definition}\label{transducer}
A transducer $G$ is a tuple $\langle \states[G],I,\Sigma,\Omega,E \rangle$.
\begin{enumerate}
\item $\states[G]$ is a finite set of states.  $I \subseteq \states[G]$ is the set of \emph{initial states}.
\item $\Sigma$ and $\Omega$ are the \emph{input alphabet} and \emph{output alphabet}, respectively -- finite sets of characters which do not contain the reserved symbol \#.
\item $E \subseteq \states[G] \times \states[G] \times (\Sigma^*\cup\{\#\}) \times \mathscr P^*(\Omega^*)$ is a finite relation called the \emph{transition relation}.  An element $e \in E$ is called a \emph{transition} with $e = \langle start(e), end(e), \trinput(e), \troutput(e) \rangle$.  If $\trinput(e) = \#$, then $e$ is called a \emph{\#-transition}.
\end{enumerate}
A transducer is said to \emph{generate} or \emph{induce} the bi-language which consists of all pairs of strings $\langle x,Y \rangle \in \Sigma^* \times \Omega^*$ such that:
\begin{enumerate}
\item $(\exists x_0, \ldots , x_n \in \Sigma^*)(x = x_0 \ldots x_n)$,
\item $(\exists e_0, \ldots , e_{n+1} \in E)(\exists q \in I)\Big((\forall i  \in \{ 1, \ldots , n \})\big(x_i = \trinput(e_i) \wedge end(e_i) = start(e_{i+1})\big) \wedge start(e_0) = q \wedge \trinput(e_{n+1}) = \#\Big)$ and 
\item there are $Y_i \in \troutput(e_i)$ for $i \leq n+1$ such that $Y = Y_0Y_1 \cdots Y_{n+1}$.
\end{enumerate}
\end{Definition}

%

This paper addresses \emph{semi-deterministic bi-languages} which are bi-languages generated by \emph{semi-deterministic transducers}.  These were defined in \cite{akra13}.  We use an equivalent formulation.

%
%
%
%
%
%

\begin{Definition}
A \emph{semi-deterministic transducer (SDT)} is a transducer with a unique initial state such that \allowbreak
\begin{enumerate}

\item $\trinput(e)\in \Sigma \cup \{\#\}$ for every transition $e$,

\item given a state, $q$, and $a\in \Sigma$, there is at most one transition, $e$, with $start(e) = q$ and $\trinput(e) = a$ and

\item given a transition, $e$, $\troutput(e)$ is a finite set of pairwise incomparable strings in $\Omega^*$ (i.e., $\troutput(e) \in \mathscr P^*(\Omega^*) \wedge (\forall X,Y \in \troutput(e))\big( X \perp Y \big)$).  
\end{enumerate}
%
A \emph{semi-deterministic bi-language} (SDBL) is a bi-language that can be generated by an SDT.
\end{Definition}

%
%

Two useful properties of SDTs follow from the definition.  First, if $e\in E$ and $\lambda \in \troutput(e)$, then $\troutput(e) = \{\lambda\}$.  Second, although there may be multiple translations of a single string, every input string follows a \emph{unique path} through an SDT.  The precise meaning of this is made clear in the next definition.  We must also note that, while SDBLs can be infinite, the image of any member or finite subset of $L$ is finite.  Thus, an SDBL is a function $f: L \rightarrow \mathscr P^*(\Omega^*)$. 

%
%
%
%
%

\begin{Definition}\label{paths-def}
Let $G$ be an SDT with input language $L$.  A \emph{path through $G$} is a string $e_0 \ldots e_k\in E^*$, where $E$ is the set of transitions, such that $start(e_{i+1}) = end(e_i)$ for $i<k$.  $G[p]$ is the collection of all outputs of $G$ that can result from following path $p$.  $p_x$ is the unique path through $G$, $e_0 \ldots e_k \in E^*$, defined by $x\in \Sigma^*$ such that $start(e_0)$ is the unique initial state of $G$, if such a path exists.  We denote the final state of the path $p_x$ by $q_x$. 
\end{Definition}

\section{Ordering maximal antichains}

When parsing sets of strings, we will often use the following operations.

\begin{Definition}\label{*-1-def} Let $S$ and $P$ be two sets of strings. 
\begin{itemize}
\item $P * S = \{x y : x\in P \wedge y\in S\}$.

\item $P^{-1}S = \{ y : (\exists x\in P)\big( x y \in S \big)  \}$.
\end{itemize}
For notational simplicity, we define $x^{-1}S = \{x\}^{-1}S$, $P^{-1}x = P^{-1}\{x\}$, $x*S = \{x\}*S$ and $P*x = P*\{x\}$ for a string $x$.
\end{Definition}


\begin{Proposition}
$*$ is associative, but is not commutative.
\end{Proposition}

\begin{proof}
Associativity follows from the associativity of concatenation.  To see that $*$ is not commutative, consider $A=\{ a \}$ and $B=\{ a,b \}$.  $A*B=\{ aa,ab \}$ and $B*A=\{ aa,ba \}$.
\end{proof}

The following definitions and results pertain to sets of strings and trees over finite alphabets.  


\begin{Definition}\label{anti-def}
Given a set of strings, $S$, we call $P \subseteq T[S]$ a \emph{maximal antichain of $S$} if $(\forall x,y\in P)\big(x\perp y \vee x = y\big)$ and $(\forall x\in S)(\exists y\in P)(y \parallel x)$.  $P$ is a \emph{valid antichain of $S$} if $P$ is a maximal antichain of $S$ and $(\forall x,y \in  P)\big(x^{-1}T[S] = y^{-1}T[S]\big)$.  We define, $\vac(S) = \{P:P \mbox{ is a valid antichain of $S$}\}$.
\end{Definition}



\begin{Example}
Consider the following set of strings over the alphabet $\{a,b\}$:
$$S = \{ a^5,a^4b,a^2ba,a^2b^2,ba^4,ba^3b,baba,bab^2,b^2a^3,b^2a^2b,b^3a,b^4 \}.$$
Graphically, we can represent $S$ as a tree where branching left indicates an $a$ and branching right indicates a $b$.  In the picture below to the right, we highlight the four valid antichains of $S$: $P_0 = \{ \lambda \}$, $P_1 = \{ a^2,ba,b^2 \}$, $P_2 = \{ a^4,a^2b,ba^3,bab,b^2a^2,b^3 \}$ and $P_3 = S$.  Note that $S$ is only a valid antichain of itself because it contains no comparable strings.  The members of the four valid antichains are connected via dotted lines in the right picture ($P_0$ has only one member and therefore includes no dotted lines).  For reference a maximal antichain that is not valid is included in the picture on the left and its members are joined with a dotted line.
\begin{figure}[H]
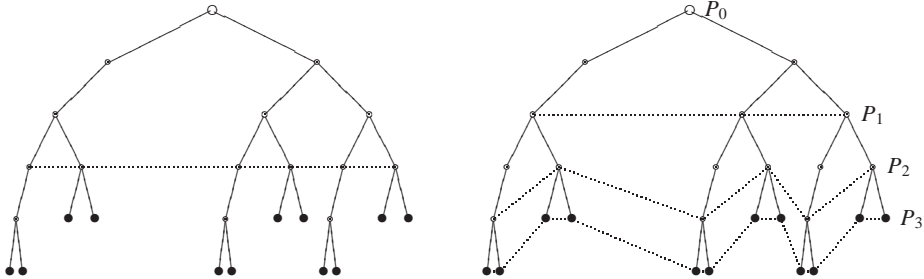

\centering
\resizebox{12.5cm}{!}{
\xy
(-35,0)*++{\xy (0,0)*+[o]=<2pt>\cir<2pt>{}="0-0";
(16,-8)*+[o]=<2pt>\hbox{}*\frm{oo}="0-1";
(-16,-8)*+[o]=<2pt>\hbox{}*\frm{oo}="1-1";
(24,-16)*+[o]=<2pt>\hbox{}*\frm{oo}="0-2";
(8,-16)*+[o]=<2pt>\hbox{}*\frm{oo}="1-2";
(-24,-16)*+[o]=<2pt>\hbox{}*\frm{oo}="3-2";
(28,-24)*+[o]=<2pt>\hbox{}*\frm{oo}="0-3";
(20,-24)*+[o]=<2pt>\hbox{}*\frm{oo}="1-3";
(12,-24)*+[o]=<2pt>\hbox{}*\frm{oo}="2-3";
(4,-24)*+[o]=<2pt>\hbox{}*\frm{oo}="3-3";
(-20,-24)*+[o]=<2pt>\hbox{}*\frm{oo}="6-3";
(-28,-24)*+[o]=<2pt>\hbox{}*\frm{oo}="7-3";
(30,-32)*+[o]=<2pt>\hbox{\textbullet}*\frm{oo}="0-4";
(26,-32)*+[o]=<2pt>\hbox{\textbullet}*\frm{oo}="1-4";
(18,-32)*+[o]=<2pt>\hbox{}*\frm{oo}="3-4";
(14,-32)*+[o]=<2pt>\hbox{\textbullet}*\frm{oo}="4-4";
(10,-32)*+[o]=<2pt>\hbox{\textbullet}*\frm{oo}="5-4";
(2,-32)*+[o]=<2pt>\hbox{}*\frm{oo}="7-4";
(-18,-32)*+[o]=<2pt>\hbox{\textbullet}*\frm{oo}="12-4";
(-22,-32)*+[o]=<2pt>\hbox{\textbullet}*\frm{oo}="13-4";
(-30,-32)*+[o]=<2pt>\hbox{}*\frm{oo}="15-4";
(19,-40)*+[o]=<2pt>\hbox{\textbullet}*\frm{oo}="6-5";
(17,-40)*+[o]=<2pt>\hbox{\textbullet}*\frm{oo}="7-5";
(3,-40)*+[o]=<2pt>\hbox{\textbullet}*\frm{oo}="14-5";
(1,-40)*+[o]=<2pt>\hbox{\textbullet}*\frm{oo}="15-5";
(-29,-40)*+[o]=<2pt>\hbox{\textbullet}*\frm{oo}="30-5";
(-31,-40)*+[o]=<2pt>\hbox{\textbullet}*\frm{oo}="31-5";
"0-0";"0-1"**\dir{-};
"0-0";"1-1"**\dir{-};
"0-1";"0-2"**\dir{-};
"0-1";"1-2"**\dir{-};
"1-1";"3-2"**\dir{-};
"0-2";"0-3"**\dir{-};
"0-2";"1-3"**\dir{-};
"1-2";"2-3"**\dir{-};
"1-2";"3-3"**\dir{-};
"3-2";"6-3"**\dir{-};
"3-2";"7-3"**\dir{-};
"0-3";"0-4"**\dir{-};
"0-3";"1-4"**\dir{-};
"1-3";"3-4"**\dir{-};
"2-3";"4-4"**\dir{-};
"2-3";"5-4"**\dir{-};
"3-3";"7-4"**\dir{-};
"6-3";"12-4"**\dir{-};
"6-3";"13-4"**\dir{-};
"7-3";"15-4"**\dir{-};
"3-4";"6-5"**\dir{-};
"3-4";"7-5"**\dir{-};
"7-4";"14-5"**\dir{-};
"7-4";"15-5"**\dir{-};
"15-4";"30-5"**\dir{-};
"15-4";"31-5"**\dir{-};
%
"0-3";"1-3"**\dir{.};
"1-3";"2-3"**\dir{.};
"2-3";"3-3"**\dir{.};
"3-3";"6-3"**\dir{.};
"6-3";"7-3"**\dir{.};
\endxy }="G1";
(40,0)*++{\xy (0,0)*+[o]=<2pt>\cir<2pt>{}="0-0";
(4,0)*+[o]=<2pt>\hbox{$P_0$}="0-0-caption";
(16,-8)*+[o]=<2pt>\hbox{}*\frm{oo}="0-1";
(-16,-8)*+[o]=<2pt>\hbox{}*\frm{oo}="1-1";
(24,-16)*+[o]=<2pt>\hbox{}*\frm{oo}="0-2";
(28,-16)*+[o]=<2pt>\hbox{$P_1$}="0-2-caption";
(8,-16)*+[o]=<2pt>\hbox{}*\frm{oo}="1-2";
(-24,-16)*+[o]=<2pt>\hbox{}*\frm{oo}="3-2";
(28,-24)*+[o]=<2pt>\hbox{}*\frm{oo}="0-3";
(32,-24)*+[o]=<2pt>\hbox{$P_2$}="0-3-caption";
(20,-24)*+[o]=<2pt>\hbox{}*\frm{oo}="1-3";
(12,-24)*+[o]=<2pt>\hbox{}*\frm{oo}="2-3";
(4,-24)*+[o]=<2pt>\hbox{}*\frm{oo}="3-3";
(-20,-24)*+[o]=<2pt>\hbox{}*\frm{oo}="6-3";
(-28,-24)*+[o]=<2pt>\hbox{}*\frm{oo}="7-3";
(30,-32)*+[o]=<2pt>\hbox{\textbullet}*\frm{oo}="0-4";
(34,-32)*+[o]=<2pt>\hbox{$P_3$}="0-4-caption";
(26,-32)*+[o]=<2pt>\hbox{\textbullet}*\frm{oo}="1-4";
(18,-32)*+[o]=<2pt>\hbox{}*\frm{oo}="3-4";
(14,-32)*+[o]=<2pt>\hbox{\textbullet}*\frm{oo}="4-4";
(10,-32)*+[o]=<2pt>\hbox{\textbullet}*\frm{oo}="5-4";
(2,-32)*+[o]=<2pt>\hbox{}*\frm{oo}="7-4";
(-18,-32)*+[o]=<2pt>\hbox{\textbullet}*\frm{oo}="12-4";
(-22,-32)*+[o]=<2pt>\hbox{\textbullet}*\frm{oo}="13-4";
(-30,-32)*+[o]=<2pt>\hbox{}*\frm{oo}="15-4";
(19,-40)*+[o]=<2pt>\hbox{\textbullet}*\frm{oo}="6-5";
(17,-40)*+[o]=<2pt>\hbox{\textbullet}*\frm{oo}="7-5";
(3,-40)*+[o]=<2pt>\hbox{\textbullet}*\frm{oo}="14-5";
(1,-40)*+[o]=<2pt>\hbox{\textbullet}*\frm{oo}="15-5";
(-29,-40)*+[o]=<2pt>\hbox{\textbullet}*\frm{oo}="30-5";
(-31,-40)*+[o]=<2pt>\hbox{\textbullet}*\frm{oo}="31-5";
"0-0";"0-1"**\dir{-};
"0-0";"1-1"**\dir{-};
"0-1";"0-2"**\dir{-};
"0-1";"1-2"**\dir{-};
"1-1";"3-2"**\dir{-};
"0-2";"0-3"**\dir{-};
"0-2";"1-3"**\dir{-};
"1-2";"2-3"**\dir{-};
"1-2";"3-3"**\dir{-};
"3-2";"6-3"**\dir{-};
"3-2";"7-3"**\dir{-};
"0-3";"0-4"**\dir{-};
"0-3";"1-4"**\dir{-};
"1-3";"3-4"**\dir{-};
"2-3";"4-4"**\dir{-};
"2-3";"5-4"**\dir{-};
"3-3";"7-4"**\dir{-};
"6-3";"12-4"**\dir{-};
"6-3";"13-4"**\dir{-};
"7-3";"15-4"**\dir{-};
"3-4";"6-5"**\dir{-};
"3-4";"7-5"**\dir{-};
"7-4";"14-5"**\dir{-};
"7-4";"15-5"**\dir{-};
"15-4";"30-5"**\dir{-};
"15-4";"31-5"**\dir{-};
%
"0-2";"1-2"**\dir{.};
"1-2";"3-2"**\dir{.};
"0-3";"3-4"**\dir{.};
"3-4";"2-3"**\dir{.};
"2-3";"7-4"**\dir{.};
"7-4";"6-3"**\dir{.};
"6-3";"15-4"**\dir{.};
"0-4";"1-4"**\dir{.};
"1-4";"6-5"**\dir{.};
"6-5";"7-5"**\dir{.};
"7-5";"4-4"**\dir{.};
"4-4";"5-4"**\dir{.};
"5-4";"14-5"**\dir{.};
"14-5";"15-5"**\dir{.};
"15-5";"12-4"**\dir{.};
"12-4";"13-4"**\dir{.};
"13-4";"30-5"**\dir{.};
"30-5";"31-5"**\dir{.};
\endxy }="G1";
\endxy
}
\caption{On the left, a maximal antichain that is not valid; on the right, all the valid antichains.}
\end{figure}
%
%
In the next figure, we focus on the valid antichain $P_1$.
\begin{figure}[H]
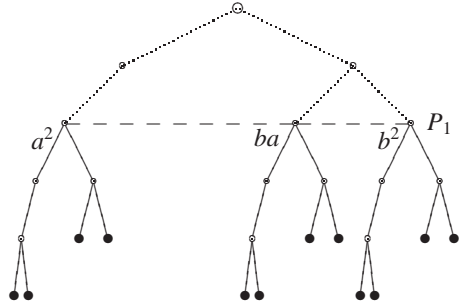

\centering
\resizebox{6cm}{!}{
\xy
(0,0)*+[o]=<2pt>\cir<2pt>{}="0-0";
(16,-8)*+[o]=<2pt>\hbox{}*\frm{oo}="0-1";
(-16,-8)*+[o]=<2pt>\hbox{}*\frm{oo}="1-1";
(24,-16)*+[o]=<2pt>\hbox{}*\frm{oo}="0-2";
(28,-16)*+[o]=<2pt>\hbox{$P_1$}="0-2-caption";
(21,-18)*+[o]=<2pt>\hbox{$b^2$}="0-2-caption2";
(8,-16)*+[o]=<2pt>\hbox{}*\frm{oo}="1-2";
(4,-18)*+[o]=<2pt>\hbox{$ba$}="1-2-caption2";
(-24,-16)*+[o]=<2pt>\hbox{}*\frm{oo}="3-2";
(-27,-18)*+[o]=<2pt>\hbox{$a^2$}="2-2-caption2";
(28,-24)*+[o]=<2pt>\hbox{}*\frm{oo}="0-3";
(20,-24)*+[o]=<2pt>\hbox{}*\frm{oo}="1-3";
(12,-24)*+[o]=<2pt>\hbox{}*\frm{oo}="2-3";
(4,-24)*+[o]=<2pt>\hbox{}*\frm{oo}="3-3";
(-20,-24)*+[o]=<2pt>\hbox{}*\frm{oo}="6-3";
(-28,-24)*+[o]=<2pt>\hbox{}*\frm{oo}="7-3";
(30,-32)*+[o]=<2pt>\hbox{\textbullet}*\frm{oo}="0-4";
(26,-32)*+[o]=<2pt>\hbox{\textbullet}*\frm{oo}="1-4";
(18,-32)*+[o]=<2pt>\hbox{}*\frm{oo}="3-4";
(14,-32)*+[o]=<2pt>\hbox{\textbullet}*\frm{oo}="4-4";
(10,-32)*+[o]=<2pt>\hbox{\textbullet}*\frm{oo}="5-4";
(2,-32)*+[o]=<2pt>\hbox{}*\frm{oo}="7-4";
(-18,-32)*+[o]=<2pt>\hbox{\textbullet}*\frm{oo}="12-4";
(-22,-32)*+[o]=<2pt>\hbox{\textbullet}*\frm{oo}="13-4";
(-30,-32)*+[o]=<2pt>\hbox{}*\frm{oo}="15-4";
(19,-40)*+[o]=<2pt>\hbox{\textbullet}*\frm{oo}="6-5";
(17,-40)*+[o]=<2pt>\hbox{\textbullet}*\frm{oo}="7-5";
(3,-40)*+[o]=<2pt>\hbox{\textbullet}*\frm{oo}="14-5";
(1,-40)*+[o]=<2pt>\hbox{\textbullet}*\frm{oo}="15-5";
(-29,-40)*+[o]=<2pt>\hbox{\textbullet}*\frm{oo}="30-5";
(-31,-40)*+[o]=<2pt>\hbox{\textbullet}*\frm{oo}="31-5";
"0-0";"0-1"**\dir{.};
"0-0";"1-1"**\dir{.};
"0-1";"0-2"**\dir{.};
"0-1";"1-2"**\dir{.};
"1-1";"3-2"**\dir{.};
"0-2";"0-3"**\dir{-};
"0-2";"1-3"**\dir{-};
"1-2";"2-3"**\dir{-};
"1-2";"3-3"**\dir{-};
"3-2";"6-3"**\dir{-};
"3-2";"7-3"**\dir{-};
"0-3";"0-4"**\dir{-};
"0-3";"1-4"**\dir{-};
"1-3";"3-4"**\dir{-};
"2-3";"4-4"**\dir{-};
"2-3";"5-4"**\dir{-};
"3-3";"7-4"**\dir{-};
"6-3";"12-4"**\dir{-};
"6-3";"13-4"**\dir{-};
"7-3";"15-4"**\dir{-};
"3-4";"6-5"**\dir{-};
"3-4";"7-5"**\dir{-};
"7-4";"14-5"**\dir{-};
"7-4";"15-5"**\dir{-};
"15-4";"30-5"**\dir{-};
"15-4";"31-5"**\dir{-};
%
"0-2";"1-2"**\dir{--};
"1-2";"3-2"**\dir{--};
\endxy
}
\caption{The identical subtrees below the elements of the valid antichain $P_1$.}
\end{figure}
Observe that the portions of the tree below each of $a^2$, $ba$ and $b^2$ are identical; the terminal nodes of all three sub-trees are $\{ a^3,a^2b,ab,b^2 \}$.  It is this equivalence of suffixes that makes $P_1$ a valid antichain.
\end{Example}

The concept of equivalence we have developed closely parallels that of Nerode equivalence \cite{nerode} in which two strings in a language are equivalent if there is no extension in the language that distinguishes the two strings.

It is interesting to note that the valid antichains in the above example have a natural linear ordering.  As we shall see in Theorem \ref{ac-linear}, this is not an artifact of the particular example, but is true of any finite set $S$.


\begin{Proposition}\label{prefix-of-prefix}
Suppose that $P$ is a valid antichain of a set of strings $S$ and $Q$ is a valid antichain of $P$, then $Q$ is a valid antichain of~$S$.
\end{Proposition}

\begin{proof}
Let $P$ be a valid antichain of a set of strings $S$ and let $Q$ be a valid antichain of $P$.  Every member of $T[S]$ is either a prefix or an extension of a member of $P$.  Since $P$ consists of incomparable strings, each member of $P$ has a member of $Q$ as a prefix.  Thus, $Q$ is a maximal antichain of $S$.  To see that $Q$ is a valid antichain, observe that if $x,y \in Q$, then $x^{-1}T[P] = y^{-1}T[P]$.  Since $z^{-1}T[S] = w^{-1}T[S]$ for all $z,w \in P$, $x^{-1}T[S] = y^{-1}T[S]$, thus $Q$ is a valid antichain.
\end{proof}


\begin{Definition}
For $P$ and $Q$, sets of strings over some common alphabet, we say that $P <_{ac} Q$ ($P$ is ``antichain less than" $Q$) if either
\begin{itemize}
\item $|P| < |Q|$, or
\item $|P| = |Q|$ and, for all $x\in P$ and $y\in Q$, if $x \parallel y$, then $x\prec y$.
\end{itemize}
\end{Definition}

We will use valid antichains to parse a set of strings as one would parse a single string into a prefix and suffix.  The validity of an antichain ensures that the corresponding suffix set is well-defined.  
%
%
%
%
%
%
%

\begin{Proposition}\label{p*ps}
Let $S$ be a finite set of incomparable strings.  If $P$ is a valid antichain of $S$, then $P * (P^{-1}S) = S$.
\end{Proposition}

\begin{proof}
Observe that, if $P$ is a valid antichain of $S$, then $T[P^{-1}S] = x^{-1}T[S]$ for all $x \in P$.  
\end{proof}

The antichain ordering ($<_{ac}$) has particularly nice properties when applied to $\vac(S)$, where $S$ is a finite set of strings.

\begin{Proposition}\label{comparable}
If $P$ and $Q$ are maximal antichains of the same finite set of strings, then there is a relation $R \subseteq P \times Q$ such that
\begin{itemize}
\item $\mbox{dom}(R) = P$,
\item $\mbox{ran}(R) = Q$,
\item $xRy \leftrightarrow x \parallel y$.
\end{itemize}

Furthermore, if $|P| = |Q|$ and $P \parallel_{ac} Q$, then $R$ is a well-defined and bijective function.
\end{Proposition}

\begin{proof}
Define $R = \{ \langle x,y \rangle : x \in P \wedge y \in Q \wedge x \parallel y \}$.  Since $P$ and $Q$ are maximal antichains, for each $x \in P$ there is $y \in Q$ such that $x \parallel y$ hence, $\mbox{dom}(R) \supseteq P$.  Similarly, for each $y \in Q$ there is an $x \in P$ such that $x \parallel y$ thus, $\mbox{ran}(R) \supseteq Q$.  By the definition of $R$, $\mbox{dom}(R) \subseteq P$, $\mbox{ran}(R) \subseteq Q$ and $xRy \leftrightarrow x \parallel y$.  If $|P| = |Q|$ and $P \parallel_{ac} Q$, then for each $x \in P$ there is a unique comparable $y \in Q$ and vice versa.  Consequently, $R$ is well-defined and bijective in this case.
\end{proof}

\begin{Theorem}\label{ac-linear}
If $S$ is a finite set of strings, then $\Big(\vac(S), <_{ac}\Big)$ is a finite linear order.
\end{Theorem}

\begin{proof}
Consider a finite set of strings, $S$, and let $T=T[S]$.  We begin by fixing $P,Q \in \vac(S)$.  We may assume that $|P| = |Q|$; if $|P| \neq |Q|$, then $P <_{ac} Q$ or $Q <_{ac} P$.  We pick an element $x\in P$ and observe that, by Proposition \ref{comparable}, there is a $y\in Q$ such that $x \parallel y$.\par

Suppose that $x=y$ and let $x'$ be any other member of $P$.  By Proposition \ref{comparable}, there is a $y' \in Q$ such that $x' \parallel y'$.  Since $P$ and $Q$ are valid antichains and $x=y$, $x'^{-1}T=x^{-1}T=y^{-1}T=y'^{-1}T$.  Given that $x'\parallel y'$, $T$ is finite and $x'^{-1}T = y'^{-1}T$ we conclude that $x' = y'$.  Now assume $x\prec y$.   In the case $y \prec x$ simply exchange the roles of $x$ and $y$.  As above, we pick $x'\in P$ and any comparable element $y'\in Q$.  Clearly $y^{-1}T$ is a strict subtree of $x^{-1}T$ and hence, $y'^{-1}T$ is a strict subtree of $x'^{-1}T$.  We conclude that~$x' \prec y'$.\par

We have shown that any two members of $\vac(S)$ are comparable.  The remaining order properties follow immediately from the definitions.
\end{proof}

While the proof of Theorem \ref{ac-linear} is quite simple, we highlight it as a theorem because it is the critical result for the applications of valid antichains that follow.  Note that $<_{ac}$ may not be a linear order on an arbitrary collection of maximal antichains. 

\begin{Corollary}
Let $S_0,S_1,S_2, \ldots$ be a sequence of finite sets.  $\bigcap_{i\in \mathbb N} \vac(S_i)$ is linearly ordered under $<_{ac}$.
\end{Corollary}

\begin{proof}
Any subset of a linear order is a linear order.  Since $\bigcap_{i\in \mathbb N} \vac(S_i) \subseteq \vac(S_0)$, the claim follows.
\end{proof}

\begin{Definition}
Given a set of strings, $S$, a finite sequence of sets of strings, $P_0, \ldots , P_n$, is a \emph{factorization of $S$} if $S = P_0 * \cdots * P_n$ and $P_i \neq \{ \lambda \}$ for $i \leq n$.  Such a factorization is said to be \emph{maximal} if, for each $i\in \mathbb N$, $\vac(P_i) = \{ \{ \lambda \}, P_i \}$.
\end{Definition}

Note that having $\vac(P_i) = \{ \{ \lambda \}, P_i \}$ for each factor, $P_i$, in a factorization is equivalent to having $P_{i+1}$ be the $<_{ac}$-least non-trivial valid antichain of $P_i^{-1}\cdots P_0^{-1} S$.

\begin{Example}\label{factorizeEx}
We consider the following set of strings:
\begin{align*}
S = \{ &a^5,a^4b,a^3ba^2,a^3bab,a^3b^2a,a^3b^3,aba^2,abab,ab^2a^2,ab^2ab,ab^3a,ab^4,ba^4,\\
&ba^3b,ba^2ba^2,ba^2bab,ba^2b^2a,ba^2b^3,b^2a^2,b^2ab,b^3a^2,b^3ab,b^4a^2,b^4ab,b^5a,b^6 \}.
\end{align*}
In the figure below, we display the tree, $T[S]$, as well as the $<_{ac}$-least non-trivial valid antichain, $P_0 = \{ a,b \}$.
\begin{figure}[H]
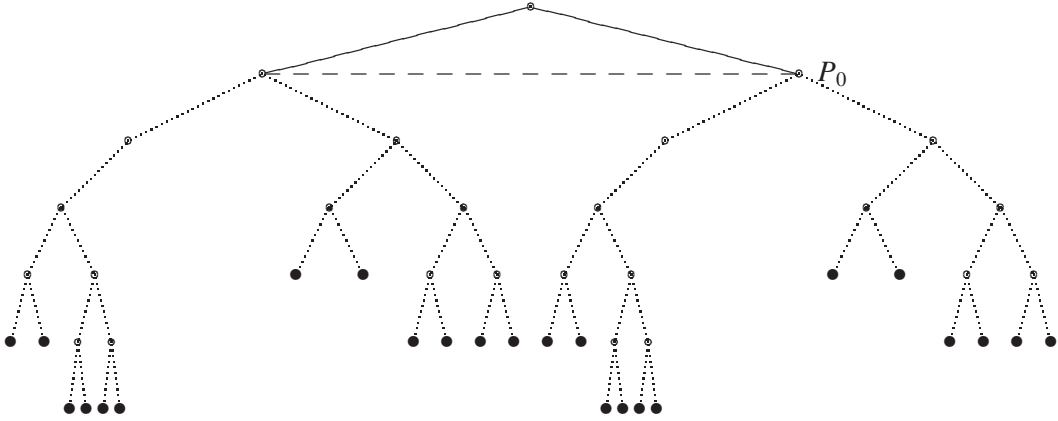

\centering
\resizebox{14cm}{!}{
\xy
(0,0)*+[o]=<2pt>\hbox{}*\frm{oo}="0-0";
(32,-8)*+[o]=<2pt>\hbox{}*\frm{oo}="0-1";
(36,-8)*+[o]=<2pt>\hbox{$P_0$}="0-2-caption";
(-32,-8)*+[o]=<2pt>\hbox{}*\frm{oo}="1-1";
(48,-16)*+[o]=<2pt>\hbox{}*\frm{oo}="0-2";
(16,-16)*+[o]=<2pt>\hbox{}*\frm{oo}="1-2";
(-16,-16)*+[o]=<2pt>\hbox{}*\frm{oo}="2-2";
(-48,-16)*+[o]=<2pt>\hbox{}*\frm{oo}="3-2";
(56,-24)*+[o]=<2pt>\hbox{}*\frm{oo}="0-3";
(40,-24)*+[o]=<2pt>\hbox{}*\frm{oo}="1-3";
(8,-24)*+[o]=<2pt>\hbox{}*\frm{oo}="3-3";
(-8,-24)*+[o]=<2pt>\hbox{}*\frm{oo}="4-3";
(-24,-24)*+[o]=<2pt>\hbox{}*\frm{oo}="5-3";
(-56,-24)*+[o]=<2pt>\hbox{}*\frm{oo}="7-3";
(60,-32)*+[o]=<2pt>\hbox{}*\frm{oo}="0-4";
(52,-32)*+[o]=<2pt>\hbox{}*\frm{oo}="1-4";
(44,-32)*+[o]=<2pt>\hbox{\textbullet}*\frm{oo}="2-4";
(36,-32)*+[o]=<2pt>\hbox{\textbullet}*\frm{oo}="3-4";
(12,-32)*+[o]=<2pt>\hbox{}*\frm{oo}="6-4";
(4,-32)*+[o]=<2pt>\hbox{}*\frm{oo}="7-4";
(-4,-32)*+[o]=<2pt>\hbox{}*\frm{oo}="8-4";
(-12,-32)*+[o]=<2pt>\hbox{}*\frm{oo}="9-4";
(-20,-32)*+[o]=<2pt>\hbox{\textbullet}*\frm{oo}="10-4";
(-28,-32)*+[o]=<2pt>\hbox{\textbullet}*\frm{oo}="11-4";
(-52,-32)*+[o]=<2pt>\hbox{}*\frm{oo}="14-4";
(-60,-32)*+[o]=<2pt>\hbox{}*\frm{oo}="15-4";
(62,-40)*+[o]=<2pt>\hbox{\textbullet}*\frm{oo}="0-5";
(58,-40)*+[o]=<2pt>\hbox{\textbullet}*\frm{oo}="1-5";
(54,-40)*+[o]=<2pt>\hbox{\textbullet}*\frm{oo}="2-5";
(50,-40)*+[o]=<2pt>\hbox{\textbullet}*\frm{oo}="3-5";
(14,-40)*+[o]=<2pt>\hbox{}*\frm{oo}="12-5";
(10,-40)*+[o]=<2pt>\hbox{}*\frm{oo}="13-5";
(6,-40)*+[o]=<2pt>\hbox{\textbullet}*\frm{oo}="14-5";
(2,-40)*+[o]=<2pt>\hbox{\textbullet}*\frm{oo}="15-5";
(-2,-40)*+[o]=<2pt>\hbox{\textbullet}*\frm{oo}="16-5";
(-6,-40)*+[o]=<2pt>\hbox{\textbullet}*\frm{oo}="17-5";
(-10,-40)*+[o]=<2pt>\hbox{\textbullet}*\frm{oo}="18-5";
(-14,-40)*+[o]=<2pt>\hbox{\textbullet}*\frm{oo}="19-5";
(-50,-40)*+[o]=<2pt>\hbox{}*\frm{oo}="28-5";
(-54,-40)*+[o]=<2pt>\hbox{}*\frm{oo}="29-5";
(-58,-40)*+[o]=<2pt>\hbox{\textbullet}*\frm{oo}="30-5";
(-62,-40)*+[o]=<2pt>\hbox{\textbullet}*\frm{oo}="31-5";
(15,-48)*+[o]=<2pt>\hbox{\textbullet}*\frm{oo}="24-6";
(13,-48)*+[o]=<2pt>\hbox{\textbullet}*\frm{oo}="25-6";
(11,-48)*+[o]=<2pt>\hbox{\textbullet}*\frm{oo}="26-6";
(9,-48)*+[o]=<2pt>\hbox{\textbullet}*\frm{oo}="27-6";
(-49,-48)*+[o]=<2pt>\hbox{\textbullet}*\frm{oo}="56-6";
(-51,-48)*+[o]=<2pt>\hbox{\textbullet}*\frm{oo}="57-6";
(-53,-48)*+[o]=<2pt>\hbox{\textbullet}*\frm{oo}="58-6";
(-55,-48)*+[o]=<2pt>\hbox{\textbullet}*\frm{oo}="59-6";
"0-0";"0-1"**\dir{-};
"0-0";"1-1"**\dir{-};
"0-1";"0-2"**\dir{.};
"0-1";"1-2"**\dir{.};
"1-1";"2-2"**\dir{.};
"1-1";"3-2"**\dir{.};
"0-2";"0-3"**\dir{.};
"0-2";"1-3"**\dir{.};
"1-2";"3-3"**\dir{.};
"2-2";"4-3"**\dir{.};
"2-2";"5-3"**\dir{.};
"3-2";"7-3"**\dir{.};
"0-3";"0-4"**\dir{.};
"0-3";"1-4"**\dir{.};
"1-3";"2-4"**\dir{.};
"1-3";"3-4"**\dir{.};
"3-3";"6-4"**\dir{.};
"3-3";"7-4"**\dir{.};
"4-3";"8-4"**\dir{.};
"4-3";"9-4"**\dir{.};
"5-3";"10-4"**\dir{.};
"5-3";"11-4"**\dir{.};
"7-3";"14-4"**\dir{.};
"7-3";"15-4"**\dir{.};
"0-4";"0-5"**\dir{.};
"0-4";"1-5"**\dir{.};
"1-4";"2-5"**\dir{.};
"1-4";"3-5"**\dir{.};
"6-4";"12-5"**\dir{.};
"6-4";"13-5"**\dir{.};
"7-4";"14-5"**\dir{.};
"7-4";"15-5"**\dir{.};
"8-4";"16-5"**\dir{.};
"8-4";"17-5"**\dir{.};
"9-4";"18-5"**\dir{.};
"9-4";"19-5"**\dir{.};
"14-4";"28-5"**\dir{.};
"14-4";"29-5"**\dir{.};
"15-4";"30-5"**\dir{.};
"15-4";"31-5"**\dir{.};
"12-5";"24-6"**\dir{.};
"12-5";"25-6"**\dir{.};
"13-5";"26-6"**\dir{.};
"13-5";"27-6"**\dir{.};
"28-5";"56-6"**\dir{.};
"28-5";"57-6"**\dir{.};
"29-5";"58-6"**\dir{.};
"29-5";"59-6"**\dir{.};
%
"0-1";"1-1"**\dir{--};
\endxy
}
\caption{A set of strings and its $<_{ac}$-least valid antichain.}
\end{figure}
%
%
%
The corresponding set of suffixes is $P_0^{-1}S = \{ a^4,a^3b,a^2ba^2,a^2bab,a^2b^2a,a^2b^3,ba^2,bab,\allowbreak b^2a^2,\allowbreak b^2ab,\allowbreak b^3a,\allowbreak b^4 \}$.  Iterating, we find the next factor is $P_1 = \{ a^2,b \}$ and its set of suffixes is $(P_0*P_1)^{-1}S = \{ a^2,ab,ba^2,bab,b^2a,b^3 \}$.
\begin{figure}[H]
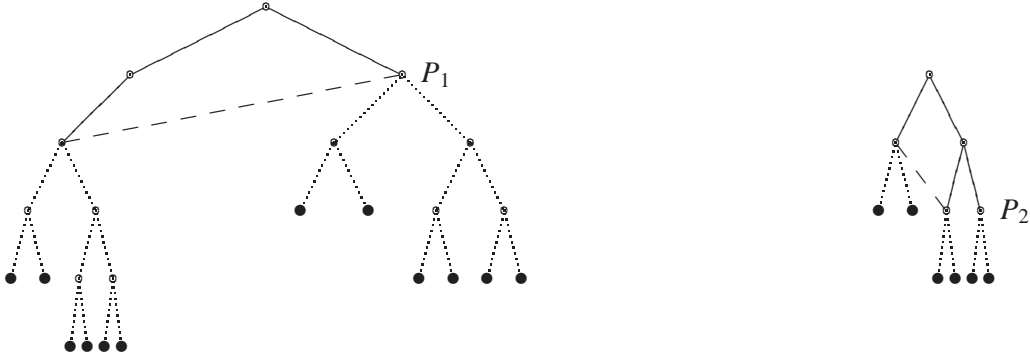

\centering
\resizebox{14cm}{!}{
\xy
(-40,0)*++{\xy (0,0)*+[o]=<2pt>\hbox{}*\frm{oo}="0-0";
(16,-8)*+[o]=<2pt>\hbox{}*\frm{oo}="0-1";
(20,-8)*+[o]=<2pt>\hbox{$P_1$}="0-2-caption";
(-16,-8)*+[o]=<2pt>\hbox{}*\frm{oo}="1-1";
(24,-16)*+[o]=<2pt>\hbox{}*\frm{oo}="0-2";
(8,-16)*+[o]=<2pt>\hbox{}*\frm{oo}="1-2";
(-24,-16)*+[o]=<2pt>\hbox{}*\frm{oo}="3-2";
(28,-24)*+[o]=<2pt>\hbox{}*\frm{oo}="0-3";
(20,-24)*+[o]=<2pt>\hbox{}*\frm{oo}="1-3";
(12,-24)*+[o]=<2pt>\hbox{\textbullet}*\frm{oo}="2-3";
(4,-24)*+[o]=<2pt>\hbox{\textbullet}*\frm{oo}="3-3";
(-20,-24)*+[o]=<2pt>\hbox{}*\frm{oo}="6-3";
(-28,-24)*+[o]=<2pt>\hbox{}*\frm{oo}="7-3";
(30,-32)*+[o]=<2pt>\hbox{\textbullet}*\frm{oo}="0-4";
(26,-32)*+[o]=<2pt>\hbox{\textbullet}*\frm{oo}="1-4";
(22,-32)*+[o]=<2pt>\hbox{\textbullet}*\frm{oo}="2-4";
(18,-32)*+[o]=<2pt>\hbox{\textbullet}*\frm{oo}="3-4";
(-18,-32)*+[o]=<2pt>\hbox{}*\frm{oo}="12-4";
(-22,-32)*+[o]=<2pt>\hbox{}*\frm{oo}="13-4";
(-26,-32)*+[o]=<2pt>\hbox{\textbullet}*\frm{oo}="14-4";
(-30,-32)*+[o]=<2pt>\hbox{\textbullet}*\frm{oo}="15-4";
(-17,-40)*+[o]=<2pt>\hbox{\textbullet}*\frm{oo}="24-5";
(-19,-40)*+[o]=<2pt>\hbox{\textbullet}*\frm{oo}="25-5";
(-21,-40)*+[o]=<2pt>\hbox{\textbullet}*\frm{oo}="26-5";
(-23,-40)*+[o]=<2pt>\hbox{\textbullet}*\frm{oo}="27-5";
"0-0";"0-1"**\dir{-};
"0-0";"1-1"**\dir{-};
"0-1";"0-2"**\dir{.};
"0-1";"1-2"**\dir{.};
"1-1";"3-2"**\dir{-};
"0-2";"0-3"**\dir{.};
"0-2";"1-3"**\dir{.};
"1-2";"2-3"**\dir{.};
"1-2";"3-3"**\dir{.};
"3-2";"6-3"**\dir{.};
"3-2";"7-3"**\dir{.};
"0-3";"0-4"**\dir{.};
"0-3";"1-4"**\dir{.};
"1-3";"2-4"**\dir{.};
"1-3";"3-4"**\dir{.};
"6-3";"12-4"**\dir{.};
"6-3";"13-4"**\dir{.};
"7-3";"14-4"**\dir{.};
"7-3";"15-4"**\dir{.};
"12-4";"24-5"**\dir{.};
"12-4";"25-5"**\dir{.};
"13-4";"26-5"**\dir{.};
"13-4";"27-5"**\dir{.};
%
"0-1";"3-2"**\dir{--};
\endxy }="G0";
%
%
%
(40,0)*++{\xy (0,0)*+[o]=<2pt>\hbox{}*\frm{oo}="0-0";
(4,-8)*+[o]=<2pt>\hbox{}*\frm{oo}="0-1";
(-4,-8)*+[o]=<2pt>\hbox{}*\frm{oo}="1-1";
(6,-16)*+[o]=<2pt>\hbox{}*\frm{oo}="0-2";
(10,-16)*+[o]=<2pt>\hbox{$P_2$}="0-2-caption";
(2,-16)*+[o]=<2pt>\hbox{}*\frm{oo}="1-2";
(-2,-16)*+[o]=<2pt>\hbox{\textbullet}*\frm{oo}="2-2";
(-6,-16)*+[o]=<2pt>\hbox{\textbullet}*\frm{oo}="3-2";
(7,-24)*+[o]=<2pt>\hbox{\textbullet}*\frm{oo}="0-3";
(5,-24)*+[o]=<2pt>\hbox{\textbullet}*\frm{oo}="1-3";
(3,-24)*+[o]=<2pt>\hbox{\textbullet}*\frm{oo}="2-3";
(1,-24)*+[o]=<2pt>\hbox{\textbullet}*\frm{oo}="3-3";
"0-0";"0-1"**\dir{-};
"0-0";"1-1"**\dir{-};
"0-1";"0-2"**\dir{-};
"0-1";"1-2"**\dir{-};
"1-1";"2-2"**\dir{.};
"1-1";"3-2"**\dir{.};
"0-2";"0-3"**\dir{.};
"0-2";"1-3"**\dir{.};
"1-2";"2-3"**\dir{.};
"1-2";"3-3"**\dir{.};
%
"0-2";"1-2"**\dir{--};
"1-2";"1-1"**\dir{--};
\endxy }="G1";
\endxy
}
\caption{$P_1$ is the $<_{ac}$-least non-trivial valid antichain of $P_0^{-1}S$ and $P_2$ is the $<_{ac}$-least non-trivial valid antichain of $(P_0 * P_1)^{-1}S$.}
\end{figure}
%
%
%
We next pick $P_2 = \{ a,ba,b^2 \}$.  Once we factor out $P_2$, all that remains is $\{ a,b \}$.  The only antichains of $\{ a,b \}$ are $\{ \lambda \}$ and $\{ a,b \}$, both of which are valid antichains.  We pick the final factor to be $P_3 = \{ a,b \}$ and conclude that $P_0*P_1*P_2*P_3$ is a maximal factorization of $S$.
\end{Example}

\begin{Corollary}\label{uniqueFact}
Up to possible reordering of commutative terms, every finite set of incomparable strings has a unique maximal factorization.
\end{Corollary}

\begin{proof}
Let $S$ be a finite set of incomparable strings.  We will apply the iterative process illustrated in Example \ref{factorizeEx} to $S$.  Define $P_0$ to be the $<_{ac}$-least non-trivial valid antichain of $S$.  If $P_{0} = S$, then the process is complete.  By Theorem \ref{ac-linear}, the choice of $P_0$ is unique.  Suppose we have defined $P_0, P_1, \ldots, P_n$.  Let $S_n = P_n^{-1} \cdots P_0^{-1} S$.  To be explicit, $S_n = P_n^{-1} (P_{n-1}^{-1} ( \cdots (P_0^{-1} S)))$.  Define $P_{n+1}$ to be the $<_{ac}$-least non-trivial valid antichain of $S_n$.  As before, the choice is unique.  If $P_{n+1} = S_n$, then the process is complete.  Otherwise, we proceed to the next iteration.

Since $\vac(S)$ is finite, the process must terminate.  The uniqueness of the factorization follows from the uniqueness of the choices made at each stage of the process.
%
\end{proof}

Observe that the interative process described above specifies a unique order for the terms of the unique maximal factorization.  When the terms are listed in the order specified by this process, we will say that the factorization is in \emph{canonical order}.

\section{Semi-Deterministic Bi-Languages}

In this section, we prove the existence of a canonical SDT for every SDBL.  Determining the canonical SDT for an SDBL is done in two phases.  First, a ``maximal'' function on prefixes of the input language is found.  Finding such a maximal function is analogous to the onwarding performed in algorithms such as OSTIA and can be loosely described as the process of moving decisions earlier in the translation process.  Second, subsets of the domain on which the function has identical outputs are conflated in a largely standard merging process.  Merging produces a finite-order equivalence relation on $T[L]$.  Using this equivalence relation, we can define the canonical SDT.

\subsection{Semi-Deterministic Functions}\label{onwarding-section}

\begin{Definition}
Let $f$ be an SDBL over $L$.  $F: T[L] \rightarrow \mathscr P^*(\Omega^*)$ is a \emph{semi-deterministic function (SDF) of $f$} if, for $x\in L$, $f(x) = F(x\upto 1)*F(x\upto 2)* \cdots *F(x)*F(x\#)$.  We define $\Pi F(x) = F(x\upto 1)*F(x\upto 2)*\cdots *F(x)$.  If $F$ and $F'$ are SDFs of $f$, we say that $F \leq_{sdf} F'$ if $\Pi F(x)$ is a valid antichain of $\Pi F'(x)$ for all $x$.  The SDF \emph{induced} by $f$ is the SDF, $F$, such that $F(x) = \{ \lambda \}$ for all $x\in T[L]$ and $F(x\#) = f(x)$ for all $x\in L$.
\end{Definition}

\begin{Example}\label{SDFexample}
Suppose that $A,B,C \subseteq \Omega^*$ are finite, non-empty and not equal to $\{\lambda\}$.  Let $\Sigma = \{a\}$ be the input alphabet.  Define an SDBL, $f$, over $L = \{ a^2 \}$ by $f(a^2) = A*B*C$.  We define two incomparable SDFs of $f$ as follows.  The first SDF: $F(\lambda) = \{ \lambda \}, F(a) = A*B, F(a^2) = \{ \lambda \}$ and $F(a^2\#) = C$.  The second SDF: $F'(\lambda) = \{ \lambda \}, F'(a) = A, F'(a^2) = B*C$ and $F'(a^2\#) = \{ \lambda \}$.  Since $\Pi F(a)$ is not a valid antichain of $\Pi F'(a)$, $F \not \leq_{sdf} F'$.  Likewise, since $\Pi F'(a^2)$ is not a valid antichain of $\Pi F(a^2)$, $F' \not \leq_{sdf} F$.
\end{Example}

Example \ref{SDFexample} demonstrates that $\leq_{sdf}$ is not a linear ordering of the SDFs of a fixed SDBL.  Nonetheless, there is a $\leq_{sdf}$-maximum SDF of $f$.

\begin{Theorem}\label{maxSDF}
If $f$ is an SDBL over $L$, then there is a $\leq_{sdf}$-maximum SDF of $f$.
\end{Theorem}

\begin{proof}
For $x \in T[L]$, let $S$ be the collection of all members of $L$ that extend $x$ and let $x_0$ be the $<_{llex}$-least member of $S$.  By Corollary \ref{uniqueFact}, for every $y \in S$ there is a unique maximal factorization of $f(y)$.  Let $P_0*\cdots * P_n$ denote the unique maximal factorization of $f(x_0)$.  Let $P_0 * \cdots * P_i$ be the longest common initial segment of all factorizations of members of $\{f(x) : x\in S\}$ when the terms of the factorizations are listed in canonical order.  We define $P^x$ to be the product of this longest common factorization.

We define $F_m(\lambda) = \{ \lambda \}$ and define $F_m$ inductively on the members of $T[L]$ in $<_{llex}$-order as follows.  Suppose we are considering $x \in T[L]$ and $F_m$ has already been defined on all $<_{llex}$-lesser members of $T[L]$.  We define $F_m(x) = (\Pi F_m(x^-))^{-1}P^x$.  If $y\in L$ and $F_m(y)$ is defined, we set $F_m(y\#) = (\Pi F_m(y))^{-1}f(y)$.

If $x \prec y$, then $\Pi F_m(x)$ is a valid antichain of $F_m(y)$ and $(F_m(y))^{-1}f(y)$ is well-defined.  Consequently, $F_m$ is a well defined function with domain $T[L]$.  If $F$ is any SDF of $f$ and $x$ is an arbitrary member of $T[L]$, then $\Pi F(x), \Pi F_m(x) \in \vac(f(x_0))$, where $x_0$ is the $<_{llex}$-least extension of $x$ in $L$.  By Theorem \ref{ac-linear}, for any $x \in T[L]$, $\Pi F(x)$ and $\Pi F_m(x)$ are $<_{ac}$-comparable.  Furthermore, $\Pi F(x), \Pi F_m(x) \in \vac(f(y))$ for all $y\succ x$.  Given the construction of $F_m$, if $F_m(x) <_{ac} F(x)$, then there must be a $y \in L$ such that $x \prec y$ and $\Pi F(x) \not\in \vac(f(y))$ -- which is not possible.  Thus, $F_m$ is a $\leq_{sdf}$-maximum SDF of $f$.
\end{proof}


\begin{Definition}
Let $f$ be an SDBL with maximal SDF $F$.  For $x \in \mbox{dom}(F)$ and $F'$ an SDF of $f$, we say that $F'$ is \emph{onward at $x$} if for all $y \in \mbox{dom}(F)$, $y \succeq x$ implies that $F'(y) = F(y)$.  If $F'$ is onward at $\lambda$, then we say that $F'$ is \emph{onward}.
\end{Definition}

In Section \ref{transQueries}, we use the concept of onwarding to build the maximal SDF from data.

\subsection{Merging}

The second phase of building a canonical form for SDTs is to define an equivalence relation on the domain of a maximum SDF.  This means identifying which paths lead to the same state.  

\begin{Definition}\label{futureOne}
Let $F$ be an SDF of $f$ over $L$ and $x\in T[L]$.  We define $\future_F[x] : x^{-1}T[L] \rightarrow R$, where $R$ is the range of $F$, such that $\future_F[x](y) = F(xy)$.  If $x,y \in \mbox{dom}(F)$, we say that $x \equiv y$ if $\future_F[x] = \future_F[y]$.  Given $x$, we define $\lllesdf{x}$ to be the $<_{llex}$-least element of $\mbox{dom}(F)$ that is equivalent to $x$.
\end{Definition}

\begin{Proposition}\label{mergeProp}$ $
\begin{enumerate}
\item \label{mergePropA} $\equiv$ is an equivalence relation on the domain of an SDF.

\item \label{mergePropB} If $x \equiv y$ and $xz,yz \in T[L]$, then $xz \equiv yz$.

\item \label{finiteEquivClasses} If $F$ is an SDF of $f$ over $L$, then there are only a finite number of $\equiv$-equivalence classes on the domain of $F$.
\end{enumerate}
\end{Proposition}

\begin{proof}
Part \ref{mergePropA} follows from the fact that equality is an equivalence relation.  Part \ref{mergePropB} follows from the definition of $\equiv$.  To prove part \ref{finiteEquivClasses}, let $G$ be an SDT that generates $f$ and let $q_x$ be a state of $G$ which can be reached by the input string $x\in T[L]$.  For any $y \in T[L]$, if $p_y$ leads to $q_x$, then $x \equiv y$ as their futures are the same.  Thus, $\equiv$ induces an equivalence relation on (hence, a partition of) the states of $G$.  Since there is at least one state in each equivalence class, the fact that $|\states[G]| < \infty$ implies that there are only finitely many equivalence classes.
\end{proof}

%
%
%


\begin{Lemma}\label{boundedFuture}
Let $F$ be an SDF of $f$ over $L$.  There is an $n$ such that for all $x,y \in T[L]$, $x \equiv y$ if and only if $\future[x]\upto x\Sigma^n = \future[y]\upto y\Sigma^n$.
\end{Lemma}

\begin{proof} 
The proof follows immediately from Proposition \ref{mergeProp}, part \ref{finiteEquivClasses}.  Since there are only a finite number of possible futures, there is a finite portion of each that uniquely identifies it.  Let $n$ be the maximum depth of the paths required to obtain the identifying portion of each future.  We have obtained the desired $n$.
\end{proof}

We can think of the identifying bounded future of an equivalence class as a sort of signature, an analogue of the famous locking sequence for Gold style learning \cite{blum-blum}.

The maximum SDF and the equivalence relation on its domain depend only on the underlying SDBL.  Thus, we have defined a machine-independent canonical form.  As a footnote, we demonstrate here how to produce an SDT from the canonical form which is unique up to isomorphism.  

\begin{Definition}\label{canonical-form-def}
Let $f$ be an SDBL, let $F_m$ be the maximum SDF for $f$ and let $\equiv$ be the equivalence relation on the domain of $F_m$.  Define a finite state machine, $G_f$, as follows:
\begin{itemize}
\item $\states[G_f] = \{ r_{\lllesdf{x}} : x\in T[L] \}$ (in other words, a set of blank states indexed by $\{ \lllesdf{x} : x \in T[L] \}$).

\item The initial state is $r_{\lambda}$.

\item $E_{G_f} = \{ \langle r_{\lllesdf{x^-}},r_{\lllesdf{x}},x(|x|-1),F_m(x) \rangle : x\in T[L] \}\cup \{ \langle r_{\lllesdf{x}},r_{\lambda},\#,F_m(x\#) \rangle : x\in L \}$
\end{itemize}
We call $G_f$ the \emph{canonical SDT for $f$}.
\end{Definition}

As noted prior to the definition, the maximum SDF depends only on the SDBL.  Thus, we are justified in calling the above SDT a canonical SDT.  Although $L$ and $T[L]$ may be infinite sets, the set of transitions, $E_{G_f}$, and the set of states, $\states[G_f]$, are finite by Proposition \ref{mergeProp}.  Also, observe that the method of defining an SDT from an SDF described in Definition \ref{canonical-form-def} can be used to define a unique SDT from any SDF.  Since every SDT also defines a unique SDF, there is a bijection between SDFs and SDTs for a given SDBL.

\begin{Theorem}
Let $f$ be an SDBL.  $G_f$ is an SDT that generates $f$.
\end{Theorem}

\begin{proof}
Clearly, $G_f$ is a finite state transducer.  If $P_0, \cdots , P_n$ are sets of incomparable strings, then $S = P_0 * \cdots * P_n$ also consists of incomparable strings.  To see this, suppose $x = x_0 \cdots x_n$ and $y = y_0 \cdots y_n$ are such that $x \prec y$ and $x_i,y_i \in P_i$ for all $i \leq n$.  If $i$ be least such that $x_i \neq y_i$, then $x_i \prec y_i$ and $P_i$ contains two comparable strings.  Thus, the outputs of all transitions of $G_f$ consist of incomparable strings, as they are factors of the elements of the range of $f$.

We must show that $G_f$ generates $f$.  $G_f$ and $f$ have the same domain.  Let $F_m$ be the maximal SDF of $f$.  If $x \in T[L]$, then $G_f[p_x] = \Pi F_m(x)$, thus, $G_f$ generates $f$.
\end{proof}

\subsection{An Example}

To illustrate the canonical form that we have now defined, we exhibit a transducer not in canonical form together with its canonical form.
\begin{figure}[H]
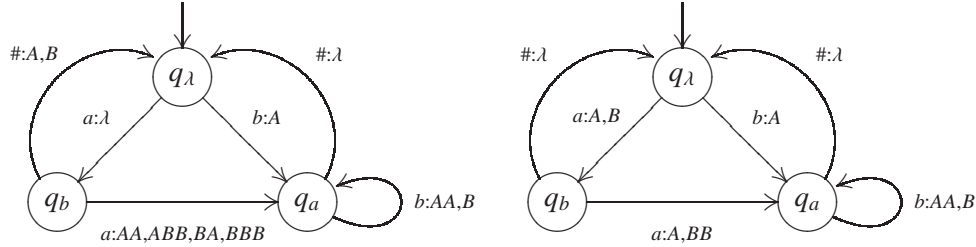

\centering
\resizebox{14cm}{!}{
\xy
(0,0)*++{\xy (0,10)*+{}="init";
(0,0)*+[o]=<20pt>\hbox{$q_{\lambda}$}*\frm{o}="in-st";
(15,-15)*+[o]=<20pt>\hbox{$q_{a}$}*\frm{o}="ri-st";
(-15,-15)*+[o]=<20pt>\hbox{$q_b$}*\frm{o}="le-st";
(30,-25)*+{}="lp-1-1";
(30,-5)*+{}="lp-1-2";
(-25,-5)*+{}="lp-2-1";
(-10,10)*+{}="lp-2-2";
(25,-5)*+{}="lp-3-1";
(10,10)*+{}="lp-3-2";
"init";"in-st"**\dir{-} ?>*\dir2{>};
"in-st";"le-st"**\dir{-} ?>*\dir2{>} ?*_!/-11pt/{_{a:\lambda}};
"in-st";"ri-st"**\dir{-} ?>*\dir2{>} ?*_!/11pt/{_{b:A}};
"le-st";"ri-st"**\dir{-} ?>*\dir2{>} ?*_!/-11pt/{_{a:AA,ABB,BA,BBB}};
"le-st";"in-st"**\crv{"lp-2-1"&"lp-2-2"} ?>*\dir2{>} ?*_!/11pt/{_{\#:A,B}};
"ri-st";"in-st"**\crv{"lp-3-1"&"lp-3-2"} ?>*\dir2{>} ?*_!/-11pt/{_{\#:\lambda}};
"ri-st";"ri-st"**\crv{"lp-1-1"&"lp-1-2"} ?>*\dir2{>} ?*_!/-15pt/{_{b:AA,B}};
\endxy }="G0";
(60,0)*++{\xy (0,10)*+{}="init";
(0,0)*+[o]=<20pt>\hbox{$q_{\lambda}$}*\frm{o}="in-st";
(15,-15)*+[o]=<20pt>\hbox{$q_a$}*\frm{o}="ri-st";
(-15,-15)*+[o]=<20pt>\hbox{$q_b$}*\frm{o}="le-st";
(30,-25)*+{}="lp-1-1";
(30,-5)*+{}="lp-1-2";
(-25,-5)*+{}="lp-2-1";
(-10,10)*+{}="lp-2-2";
(25,-5)*+{}="lp-3-1";
(10,10)*+{}="lp-3-2";
"init";"in-st"**\dir{-} ?>*\dir2{>};
"in-st";"le-st"**\dir{-} ?>*\dir2{>} ?*_!/-11pt/{_{a:A,B}};
"in-st";"ri-st"**\dir{-} ?>*\dir2{>} ?*_!/11pt/{_{b:A}};
"le-st";"ri-st"**\dir{-} ?>*\dir2{>} ?*_!/-11pt/{_{a:A,BB}};
"le-st";"in-st"**\crv{"lp-2-1"&"lp-2-2"} ?>*\dir2{>} ?*_!/11pt/{_{\#:\lambda}};
"ri-st";"in-st"**\crv{"lp-3-1"&"lp-3-2"} ?>*\dir2{>} ?*_!/-11pt/{_{\#:\lambda}};
"ri-st";"ri-st"**\crv{"lp-1-1"&"lp-1-2"} ?>*\dir2{>} ?*_!/-15pt/{_{b:AA,B}};
\endxy }="G1";
\endxy
}
\caption{An SDBL not in canonical form (left) and in canonical form (right).}
\end{figure}

\section{The learning models}

There are two principal learning models in grammatical inference: identification in the limit \cite{gold67} and PAC-learning \cite{vali84}.  Each of these models admits variants depending on what additional sources of information are provided. In order to learn semi-deterministic transducers, we use queries \cite{angl87a} as an additional resource.  These queries are very limited; the oracle will be interrogated about a possible translation pair and the oracle will return either a \textsc{true} or \textsc{false}.

\begin{Definition}
Let $f$ be a bi-language.  The translation query $[x,Y]_f$ returns $\textsc{true}$ if $Y\in f(x)$ and $\textsc{false}$ otherwise.  We call this oracle $[f]$.  Where it is clear from context, we will write $[x,Y]$ instead of $[x,Y]_f$.
\end{Definition}
Equivalently, the oracle answers membership queries about the graph of the bi-language.  We also prove that learning is not possible without queries.  The precise definition of learning we use is adapted from the one used in \cite{higu97}:

\begin{Definition}
An algorithm, $A$, \emph{polynomial identifies in the limit with translation queries} a class of transducers, $\mathscr C$, if for any $G \in \mathscr C$ there is a set, $CS_G$, such that on any $\mathcal D \supseteq CS_G$ contained in the bi-language induced by $G$, $A$ outputs a $G'$ equivalent to $G$.  The algorithm must converge within a polynomial amount of time in $|\mathcal D|$ and $|G|$; $|CS_G|$ must be polynomial in $|G|$.  $|G|$, $|\mathcal D|$ and $|CS_G|$ denote the number of bits required to encode the objects $G$, $\mathcal D$ and $CS_G$, respectivly.
\end{Definition}

Note that in the above definition the number of calls to the oracle is also bounded by the overall complexity of the algorithm and is therefore polynomial in the size of the sample.

For Theorem \ref{not-learnable-thm}, we use a different model of learning: identification in the limit from positive data.  We give the definition below.

\begin{Definition}
An algorithm, $A$, \emph{identifies in the limit from positive data} a class of transducers, $\mathscr C$, if for any $G \in \mathscr C$ and any infinite enumeration of the bi-language induced by $G$, the algorithm $A$ outputs a finite number of distinct transducers on the initial segments of the enumeration.  The only transducer that is output infinitely many times must be equivalent to $G$.
\end{Definition}

\section{SDBLs are not learnable}

We assume domain knowledge (i.e., access to the characteristic function of the input language).  In the proof of the following theorem, we encode a standard example of a ``topological" failure of identification in the limit.  In particular, we encode the family $\mathcal H = \{\mathbb N\} \cup \{A \subseteq \mathbb N : |A|<\infty\}$ into a sequence of SDTs.

\begin{Definition}
Let $f$ be a bi-language.  We define $DK_f$ to be the oracle that, when asked about $x$, returns a boolean value $DK_f(x)$.  If $DK_f(x) = \textsc{true}$, then $x$ is in the input language of $f$ (in other words, the domain of $f$).  Otherwise, $x$ is not in the input language of $f$.  An algorithm which has access to $DK_f$ is said to have domain knowledge about $f$.
\end{Definition}



\begin{Theorem}\label{not-learnable-thm}
There is a collection of SDBLs, $\mathcal C$, such that no algorithm can identify $\mathcal C$ in the limit from positive data, even given domain knowledge of each member of $\mathcal C$.
\end{Theorem}

\begin{proof}
To avoid degenerate cases, we assume the output alphabet has at least two characters, $A$ and $B$, and the input alphabet has at least one character, $a$.  We exhibit a sequence of SDTs, $\{G_i\}_{i\in\mathbb N}$, such that no program can successfully learn every member of the sequence.  In the following graphical representation of $\{G_i\}_{i\in\mathbb N}$ we omit the \#-transitions, instead indicating terminal nodes with a double border.
\begin{figure}[H]
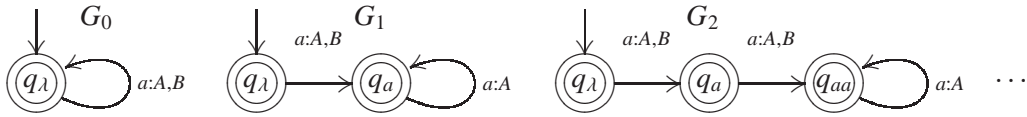

\centering
\resizebox{14cm}{!}{
\xy
(0,8)*+{G_0};
(0,0)*++{\xy (0,10)*+{}="init";
(0,0)*+[o]=<20pt>\hbox{$q_\lambda$}*\frm{oo}="in-st";
(15,-10)*+{}="lp-1-1";
(15,10)*+{}="lp-1-2";
"init";"in-st"**\dir{-} ?>*\dir2{>};
"in-st";"in-st"**\crv{"lp-1-1"&"lp-1-2"} ?>*\dir2{>} ?*_!/-11pt/{_{a:A,B}}; 
\endxy }="G0";
(33,8)*+{G_1};
(33,0)*++{\xy (0,10)*+{}="init";
(0,0)*+[o]=<20pt>\hbox{$q_\lambda$}*\frm{oo}="in-st";
(15,0)*+[o]=<20pt>\hbox{$q_a$}*\frm{oo}="a-st";
(30,-10)*+{}="lp-1-1";
(30,10)*+{}="lp-1-2";
"init";"in-st"**\dir{-} ?>*\dir2{>};
"in-st";"a-st"**\dir{-} ?>*\dir2{>} ?*_!/15pt/{_{a:A,B}};
"a-st";"a-st"**\crv{"lp-1-1"&"lp-1-2"} ?>*\dir2{>} ?*_!/-8pt/{_{a:A}}; 
\endxy }="G1";
(73,8)*+{G_2};
(110,0)*+{\cdots};
(80,0)*++{\xy (0,10)*+{}="init";
(0,0)*+[o]=<20pt>\hbox{$q_\lambda$}*\frm{oo}="in-st";
(15,0)*+[o]=<20pt>\hbox{$q_a$}*\frm{oo}="a-st";
(30,0)*+[o]=<20pt>\hbox{$q_{aa}$}*\frm{oo}="aa-st";
(45,-10)*+{}="lp-1-1";
(45,10)*+{}="lp-1-2";
"init";"in-st"**\dir{-} ?>*\dir2{>};
"in-st";"a-st"**\dir{-} ?>*\dir2{>} ?*_!/15pt/{_{a:A,B}};
"a-st";"aa-st"**\dir{-} ?>*\dir2{>} ?*_!/15pt/{_{a:A,B}};
"aa-st";"aa-st"**\crv{"lp-1-1"&"lp-1-2"} ?>*\dir2{>} ?*_!/-8pt/{_{a:A}}; 
\endxy }="G2";
\endxy
}
\caption{A sequence SDTs that cannot be identified in the limit from positive data.  Transitions are labelled with the input string they read and the set of possible output strings; for example, a transition $e$ labelled with $a:A,B$ has the property that $\trinput(e) = a$ and $\troutput(e) = \{A,B\}$.}
\end{figure}

Let $f_i$ be the SDBL generated by the SDT $G_i$.  Fix any learning algorithm and let $M$ be the function such that, given data $\mathcal D$, the hypothesis made by the learning algorithm is $M(\mathcal D)$.  We inductively define an enumeration of a bi-language generated by some member of the sequence, $\{G_i\}_{i\in\mathbb N}$.  Define $X_i = \langle a^i,A^i \rangle \langle a^i,B^i \rangle$ and $X_i^j = \langle a^j,A^j \rangle \langle a^{j+1},A^{j+1} \rangle \cdots \langle a^{j+i},A^{j+i} \rangle$.  Let $n_1$ be least such that $M(X_1 X_{n_1}^1)$ codes $G_1$.  If no such $n_1$ exists, then there is an enumeration of $f_1$ which the chosen algorithm fails to identify.  Thus, without loss of generality, we may assume such an $n_1$ exists.  Similarly, we pick $n_2$ to be least such that $M(X_1 X_{n_1}^1 X_2 X_{n_2}^2)$ codes $G_2$.  Proceeding in this fashion, either we reach a stage where some $n_k$ cannot be found and the algorithm has failed to learn $f_k$ or we have built an enumeration of $G_0$ on which the algorithm changes its hypothesis an infinite number of times.  In either case, learning has failed.  $\mathcal C = \{ f_i : i\in \mathbb N \}$ is the desired collection of SDBLs.
\end{proof}

\section{Learning with translation queries}\label{transQueries}

In the remainder of the paper, we exhibit an algorithm that can learn any SDBL, $f$, in the limit, provided the algorithm has access to the oracles $DK_f$ and $[f]$.  We present the algorithms that witness the learnability of SDBLs and summarize the result in Theorem \ref{learnableThm}.

%

\subsection{The characteristic sample}\label{sec-char-sample}

The characteristic sample must contain sufficient data to unambiguously perform two operations: onwarding and merging.  Throughout this section $f$ is an SDBL over $L$ and $G$ is the canonical SDT that generates $f$.  We define $\hat x$ to be the $<_{llex}$-least member of $L$ that extends $x$.  We now proceed to define the characteristic sample for $f$, denoted $CS_f$.  We will make extensive use of $p_x$, $q_x$ and $G[p_x]$ in this section (see Definition \ref{paths-def}).


The first component of the characteristic sample provides the data required to recognize which maximal antichains of a set of translations are not valid.  In order to illustrate the concept, consider $f(a\#)$, the translations along a path involving only one non-\# transition.  Let $X$ be the $<_{llex}$-least member of $f(a\#)$.  Every maximal antichain of $f(a\#)$ contains a prefix of $X$ and every prefix of $X$ is a member of at most one element of $\vac(f(a\#))$.  If $X_0$ is a prefix of $X$ that is not in a valid antichain, then there is a $Z \in f(a\#)$ such that for any $Z_0 \prec Z$, either
\begin{enumerate}
\item there is a $Z_1$ such that $Z_0 Z_1 \in f(a\#)$ and $X_0 Z_1 \not\in f(a\#)$, or

\item there is a $X_1$ such that $X_0 X_1 \in f(a\#)$ and $Z_0 X_1 \not\in f(a\#)$.
\end{enumerate}
In other words, $X_0$ and $Z_0$ have different futures.  Thus, for each prefix which is not an element of a valid antichain, there is a translation pair that witnesses this fact.  The following figure illustrates the two cases with the possible witnessing strings marked by dashed lines.
\begin{figure}[H]
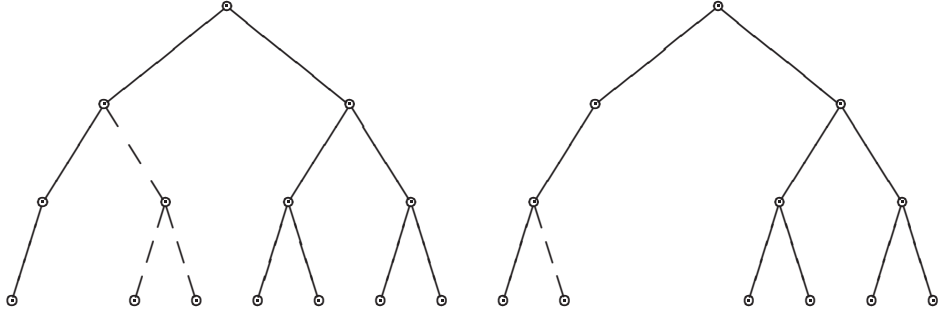

\centering
\vspace{-3em}
\resizebox{12.5cm}{!}{
\xy
(0,0)*+[o]=<2pt>\hbox{}*\frm{}="0-0";
(20,-8)*+[o]=<2pt>\hbox{}*\frm{oo}="0-1";
(-20,-8)*+[o]=<2pt>\hbox{}*\frm{oo}="1-1";
(30,-16)*+[o]=<2pt>\hbox{}*\frm{oo}="0-2";
(10,-16)*+[o]=<2pt>\hbox{}*\frm{oo}="1-2";
(-10,-16)*+[o]=<2pt>\hbox{}*\frm{oo}="2-2";
(-30,-16)*+[o]=<2pt>\hbox{}*\frm{oo}="3-2";
(35,-24)*+[o]=<2pt>\hbox{}*\frm{oo}="0-3";
(25,-24)*+[o]=<2pt>\hbox{}*\frm{oo}="1-3";
(5,-24)*+[o]=<2pt>\hbox{}*\frm{oo}="3-3";
(-5,-24)*+[o]=<2pt>\hbox{}*\frm{oo}="4-3";
(-15,-24)*+[o]=<2pt>\hbox{}*\frm{oo}="5-3";
(-25,-24)*+[o]=<2pt>\hbox{}*\frm{oo}="6-3";
(-35,-24)*+[o]=<2pt>\hbox{}*\frm{oo}="7-3";
(37.5,-32)*+[o]=<2pt>\hbox{}*\frm{oo}="0-4";
(32.5,-32)*+[o]=<2pt>\hbox{}*\frm{oo}="1-4";
(27.5,-32)*+[o]=<2pt>\hbox{}*\frm{oo}="2-4";
(22.5,-32)*+[o]=<2pt>\hbox{}*\frm{oo}="3-4";
(7.5,-32)*+[o]=<2pt>\hbox{}*\frm{oo}="6-4";
(2.5,-32)*+[o]=<2pt>\hbox{}*\frm{oo}="7-4";
(-2.5,-32)*+[o]=<2pt>\hbox{}*\frm{oo}="8-4";
(-7.5,-32)*+[o]=<2pt>\hbox{}*\frm{oo}="9-4";
(-12.5,-32)*+[o]=<2pt>\hbox{}*\frm{oo}="10-4";
(-17.5,-32)*+[o]=<2pt>\hbox{}*\frm{oo}="11-4";
(-22.5,-32)*+[o]=<2pt>\hbox{}*\frm{oo}="12-4";
(-27.5,-32)*+[o]=<2pt>\hbox{}*\frm{oo}="13-4";
(-37.5,-32)*+[o]=<2pt>\hbox{}*\frm{oo}="15-4";
"0-1";"0-2"**\dir{-};
"0-1";"1-2"**\dir{-};
"1-1";"2-2"**\dir{-};
"1-1";"3-2"**\dir{-};
"0-2";"0-3"**\dir{-};
"0-2";"1-3"**\dir{-};
"1-2";"3-3"**\dir{-};
"2-2";"4-3"**\dir{-};
"2-2";"5-3"**\dir{-};
"3-2";"6-3"**\dir{--};
"3-2";"7-3"**\dir{-};
"0-3";"0-4"**\dir{-};
"0-3";"1-4"**\dir{-};
"1-3";"2-4"**\dir{-};
"1-3";"3-4"**\dir{-};
"3-3";"6-4"**\dir{--};
"3-3";"7-4"**\dir{-};
"4-3";"8-4"**\dir{-};
"4-3";"9-4"**\dir{-};
"5-3";"10-4"**\dir{-};
"5-3";"11-4"**\dir{-};
"6-3";"12-4"**\dir{--};
"6-3";"13-4"**\dir{--};
"7-3";"15-4"**\dir{-};
\endxy
}
\caption{Two ways in which different futures might be witnessed.  In both cases, it is easy to verify that the futures are different using translation queries.}
\end{figure}
To describe the required information in the general case, let $x_0, \ldots , x_k$ enumerate the minimal paths to each of the states of $G$.  Let $x_0, \ldots , x_n$ enumerate $x_0, \ldots x_k$ together with all possible one-step extensions of the paths $x_0, \ldots , x_k$.  Note that $n$ is bounded by $|\states[G]| + |\states[G]||E|$, where $E$ is the transition relation for $G$.  Fix $i \leq n$.  If $|x_i| > 0$, let $P$ be the $<_{ac}$-greatest antichain that is a member of $\vac(f(x_i^- y))$ for all strings $y$ such that $x_i^- y \in L$; if $|x_i| = 0$, define $P = \{ \lambda \}$.  Define $X$ to be the $<_{llex}$-least member of $P^{-1} f(\hat{x_i})$.  For each $X_0 \prec X$ that is not a member of a valid antichain of $P^{-1} f(\hat{x_i})$, there is a $Y \in P^{-1} f(\hat{x_i})$ no prefix of which has the same future in $P^{-1} f(\hat{x_i})$ as $X_0$ and there is a translation in $f(\hat{x_i})$ witnessing the different futures.  We denote the set of such witnessing translation pairs, one for each prefix of $X$ not in a valid antichain, by $S_i$.  Let $Z$ be the $<_{llex}$-least member of $P$.  Let $N_0(x_i) = \{ \langle \hat{x_i}, ZX \rangle \} \cup S_i$ and define $N_0(f) = \bigcup_{i\leq n} N_0(x_i)$.  Observe that $N_0(f)$ is polynomial in the size of $G$.

Consider $x\in T[L]$.  Let $\vac = \bigcap_{x \prec y \in L} \vac(f(y))$.  For each $P \in \vac(f(x)) \setminus \vac$, observe that there is an example that witnesses the fact that $P$ is not in $\vac$.  Such examples demonstrate violations of either the maximality or the validity of the given antichain.  In either case, the witness is a single element of the graph of $f$ (a paired string and translation).  Since $\vac(f(x))$ is finite, the number of examples needed to eliminate all incorrect maximal antichains is also finite.  We define $N_1(x)$ to be the set which consists of exactly one example for each member of $\vac(f(x)) \setminus \vac$.  For the sake of a unique definition, we assume that we always choose the $<_{llex}$-least example -- although this is not essential.  We can now define the second component of $CS_f$:  $N_1(f) = \bigcup_{q\in \states[G]} N_1(\hat{x_q})$.

$N_0$ and $N_1$ are required to perform onwarding correctly.  In order to perform merges, we must include enough data to identify the equivalence classes of states whose futures are the same.  There are two ways in which the futures may differ:
\begin{enumerate}
\item there is a string, $z$, such that $xz \in L$, but $yz \not\in L$ or

\item for $X \in G[p_x]$ and $Y \in G[p_y]$, there are $z$ and $Z$ such that $XZ \in G[p_{xz}]$, but $YZ \not\in G[p_{yz}]$.
\end{enumerate}
For each member of $\states[G]$ there is a finite collection of examples which uniquely identify the state.  Let $N_2(q_x)$ be a canonically chosen collection of such examples for $q_x$.  Let $e$ be a transition and $\hat{p}$ be the $<_{llex}$-least path starting at the initial state, ending with a \#-transition and including $e$.  Define $N_2^*(e)$ to be the set of those translations of $\hat{p}$ each of which uses a different output of the transition $e$ and is $<_{llex}$-least amongst the translations of $\hat{p}$ that use that output.  $|N_2^*(e)| = |\troutput(e)|$.  We define the final component of $CS_f$ as follows.
\[
N_2(f) = \bigcup_{x \in W} N_2(q_x) \cup \bigcup_{e \in E_G} N_2^*(e),
\]
where $W$ consists of the minimal paths to each state of $G$ as well as all paths that are immediate extensions of those paths.
\begin{Definition}
For an SDBL, $f$, we define the characteristic sample of $f$, $CS_f = N_0(f) \cup N_1(f) \cup N_2(f)$.
\end{Definition}

\subsection{Algorithms}

In all the algorithms that follow, loops over prefixes of a string will proceed in order of increasing length.  Also, when a subroutine returns multiple outputs (e.g., returns all the elements of an array) we assume that an appropriate loop is executed to load the returned values into the selected variables in the main program.

\subsubsection{Initializing the transducer}

\begin{Definition}
Given a string $x$ over the input alphabet of an SDT $G$, we say that $G$ is \emph{tree-like below $x$} if every path which begins at $q_x$ ends at a state which is the end state of exactly one transition.  These states are called the \emph{states below $x$}.  $G$ is said to be \emph{tree-like} if it is tree-like below its unique initial state.
\end{Definition}

Consider a dataset, $\mathcal D$.  We define an initial transducer by creating a state for every member of $T[\mbox{dom}(\mathcal D)]$.  A tree-like transducer is produced where all transitions output only $\lambda$ except for the \#-transitions at members of $\mbox{dom}(\mathcal D)$.  All outputs in the dataset are assigned to the \#-transitions.  


\begin{algorithm}[H]\label{alg-initial}
\caption{Forming the initial tree-like transducer (INITIAL)}
\KwData{A finite collection of translation pairs, $\mathcal D$.}
\KwResult{A tree-like SDT, $G_{\mathcal D}$.}
\For{$\langle x, X \rangle \in \mathcal D$}{
	$\states[G_{\mathcal D}] \cup \{r_x\} \rightarrow \states[G_{\mathcal D}]$ \\
	$E_{G_{\mathcal D}} \cup \{ e_x^{\#} = \langle r_x, r_{\lambda}, \#, X \rangle \} \rightarrow E_{G_{\mathcal D}}$\\
	\If {$x \neq \lambda$}{
		\For{$y \prec x$}{
			$\states[G_{\mathcal D}] \cup \{r_y\} \rightarrow \states[G_{\mathcal D}]$ \\
			$E_{G_{\mathcal D}} \cup \{ e_y = \langle r_{y^-}, r_y, y(|y|-1), \lambda \rangle \} \rightarrow E_{G_{\mathcal D}}$ \\
		}
	}
}
\Return $G_{\mathcal D}$
\end{algorithm}

The transducer that results from a run of Algorithm \ref{alg-initial} recognizes the translations in $\mathcal D$ and no other translations.

\subsubsection{Generating an array of all valid antichains}\label{genVacSection}

In order to simplify the presentation of the algorithms, we will not include the algorithms for several simple functions.  In particular, we will assume that $\lexorder(A)$ takes an array, $A$, as an input and returns an array with the same contents as $A$, but in lexicographic order.  $\llexorder(A)$ performs the same function, but for the $<_{llex}$-ordering.  $\lexleast$ and $\llexleast$ will be applied to sets and arrays and will return the $<_{lex}$- and $<_{llex}$-least member, respectively.  For sets of strings $P$ and $S$, we will use the operations $P^{-1}S$ and $P*S$ as built-in arithmetic operations.  Given an input string, $x$, output strings, $Z$ and $W$, and a set of translation pairs, $\mathcal D$, the function $COMPARE(x,Z,W,\mathcal D)$ returns \textsc{true} if, for every $\langle x,ZR \rangle, \langle x,WS \rangle \in \mathcal D$, the queries $[x,WR]_f$ and $[x,ZS]_f$ return values of \textsc{true}.  Otherwise, $COMPARE(x,Z,W,\mathcal D)$ returns \textsc{false}.  Applying the same notation used above, if $x$ is an input string, then $\hat{x}$ is the $<_{llex}$-least member of $L$ extending $x$.  Using these functions, we define an algorithm to create a list of all valid antichains when considering the tree of outputs of a single input string.

\begin{algorithm}[H]\label{VAC-alg}
\caption{List the valid antichains (VAC)}
\KwData{A finite collection of translation pairs, $\mathcal D$; $x \in L$; $X_{\ell}$, the current least translation prefix for $x$.}
\KwResult{An array, $A$, of all maximal antichains of the translations of $x$ in $\mathcal D$ which extend $X_\ell$ and are not provably invalid.}

$X_{\ell}^{-1}\{ Y : Y \succ X_{\ell} \wedge \langle x,Y \rangle \in \mathcal D \} \rightarrow T$\\
$\llexleast(T) \rightarrow Z$\\
\For{$W \prec Z$}{
	$W \rightarrow AC[0]$\\
	\For{$R\in T \wedge R \neq Z$}{
		\For{$V \prec R$}{
			$COMPARE(x,X_{\ell}W,X_{\ell}V,\mathcal D) \rightarrow status$\\
			\If{$status = \textsc{true}$}{
				$V \rightarrow AC[|AC|]$\\
				\textbf{break}\\
			}
		}
		\If{$status = \textsc{false}$}{
			\textbf{break}\\
		}
	}
	\If{$status = \textsc{true}$}{
		$AC \rightarrow A[|A|]$\\
	}
}
\Return{$A$}\\

\end{algorithm}
\begin{wrapfigure}{r}{0.4\textwidth}
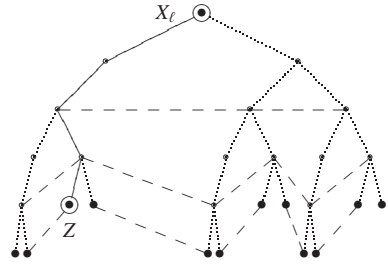

\centering
\resizebox{5cm}{!}{
\xy
(0,0)*+[o]=<8pt>\hbox{\textbullet}*\frm{oo}="0-0";
(-6,0)*+[o]=<2pt>\hbox{$X_{\ell}$}="13-4-caption";
(16,-8)*+[o]=<2pt>\hbox{}*\frm{oo}="0-1";
(-16,-8)*+[o]=<2pt>\hbox{}*\frm{oo}="1-1";
(24,-16)*+[o]=<2pt>\hbox{}*\frm{oo}="0-2";
(8,-16)*+[o]=<2pt>\hbox{}*\frm{oo}="1-2";
(-24,-16)*+[o]=<2pt>\hbox{}*\frm{oo}="3-2";
(28,-24)*+[o]=<2pt>\hbox{}*\frm{oo}="0-3";
(20,-24)*+[o]=<2pt>\hbox{}*\frm{oo}="1-3";
(12,-24)*+[o]=<2pt>\hbox{}*\frm{oo}="2-3";
(4,-24)*+[o]=<2pt>\hbox{}*\frm{oo}="3-3";
(-20,-24)*+[o]=<2pt>\hbox{}*\frm{oo}="6-3";
(-28,-24)*+[o]=<2pt>\hbox{}*\frm{oo}="7-3";
(30,-32)*+[o]=<2pt>\hbox{\textbullet}*\frm{oo}="0-4";
(26,-32)*+[o]=<2pt>\hbox{\textbullet}*\frm{oo}="1-4";
(18,-32)*+[o]=<2pt>\hbox{}*\frm{oo}="3-4";
(14,-32)*+[o]=<2pt>\hbox{\textbullet}*\frm{oo}="4-4";
(10,-32)*+[o]=<2pt>\hbox{\textbullet}*\frm{oo}="5-4";
(2,-32)*+[o]=<2pt>\hbox{}*\frm{oo}="7-4";
(-18,-32)*+[o]=<2pt>\hbox{\textbullet}*\frm{oo}="12-4";
(-22,-32)*+[o]=<8pt>\hbox{\textbullet}*\frm{oo}="13-4";
(-22,-36)*+[o]=<2pt>\hbox{$Z$}="13-4-caption";
(-30,-32)*+[o]=<2pt>\hbox{}*\frm{oo}="15-4";
(19,-40)*+[o]=<2pt>\hbox{\textbullet}*\frm{oo}="6-5";
(17,-40)*+[o]=<2pt>\hbox{\textbullet}*\frm{oo}="7-5";
(3,-40)*+[o]=<2pt>\hbox{\textbullet}*\frm{oo}="14-5";
(1,-40)*+[o]=<2pt>\hbox{\textbullet}*\frm{oo}="15-5";
(-29,-40)*+[o]=<2pt>\hbox{\textbullet}*\frm{oo}="30-5";
(-31,-40)*+[o]=<2pt>\hbox{\textbullet}*\frm{oo}="31-5";
"0-0";"0-1"**\dir{.};
"0-0";"1-1"**\dir{-};
"0-1";"0-2"**\dir{.};
"0-1";"1-2"**\dir{.};
"1-1";"3-2"**\dir{-};
"0-2";"0-3"**\dir{.};
"0-2";"1-3"**\dir{.};
"1-2";"2-3"**\dir{.};
"1-2";"3-3"**\dir{.};
"3-2";"6-3"**\dir{-};
"3-2";"7-3"**\dir{.};
"0-3";"0-4"**\dir{.};
"0-3";"1-4"**\dir{.};
"1-3";"3-4"**\dir{.};
"2-3";"4-4"**\dir{.};
"2-3";"5-4"**\dir{.};
"3-3";"7-4"**\dir{.};
"6-3";"12-4"**\dir{.};
"6-3";"13-4"**\dir{-};
"7-3";"15-4"**\dir{.};
"3-4";"6-5"**\dir{.};
"3-4";"7-5"**\dir{.};
"7-4";"14-5"**\dir{.};
"7-4";"15-5"**\dir{.};
"15-4";"30-5"**\dir{.};
"15-4";"31-5"**\dir{.};
%
"0-2";"1-2"**\dir{--};
"1-2";"3-2"**\dir{--};
"0-3";"3-4"**\dir{--};
"3-4";"2-3"**\dir{--};
"2-3";"7-4"**\dir{--};
"7-4";"6-3"**\dir{--};
"6-3";"15-4"**\dir{--};
"0-4";"1-4"**\dir{--};
"1-4";"6-5"**\dir{--};
"6-5";"7-5"**\dir{--};
"7-5";"4-4"**\dir{--};
"4-4";"5-4"**\dir{--};
"5-4";"14-5"**\dir{--};
"14-5";"15-5"**\dir{--};
"15-5";"12-4"**\dir{--};
"12-4";"13-4"**\dir{--};
"13-4";"30-5"**\dir{--};
"30-5";"31-5"**\dir{--};
\endxy
}
\caption{$X_{\ell}$ is the least translation prefix and $Z$ is the least translation}
\end{wrapfigure}
%
%
%
%
One of the inputs of Algorithm \ref{VAC-alg} is the ``current least translation prefix of $x$''.  The current translation prefix will converge to the $<_{llex}$-least output string generated along the unique path corresponding to $x$.  $X_{\ell}$ provides a canonical output prefix for testing outputs using translation queries.  The first step of Algorithm \ref{VAC-alg} restricts $\mathcal D$ to the tree of translation pairs whose second component extends the least translation prefix.  Every antichain of the tree must contain a prefix of the $<_{llex}$-least member of the tree.  Because of the linear ordering of the valid antichains (see Theorem \ref{ac-linear}), there is at most one valid antichain for each prefix of the least member of the tree.  COMPARE is used to look for matching nodes to form valid antichains.  As can be seen in the figure, all valid antichains include prefixes of the $<_{llex}$-least member and no two valid antichains contain the same prefix.  This provides both a bound on the number of valid antichains and a convenient method to search for the valid antichains.

We formalize the above intuition in the proof of the following lemma. 

\begin{Lemma}
Let $\mathcal D$ be a finite set consistent with an SDBL $f$ over $L$ with canonical transducer $G$ and $x \in L$.  Suppose $CS_f \subseteq \mathcal D$, $x$ is $<_{llex}$-least among $y \in L$ such that $q_y = q_x$, $X_{\ell}$ is the least translation prefix of $x$ and $X$ is the $<_{llex}$-least member of $f(x)$.   Given inputs $\mathcal D$, $x$ and $X_{\ell}$ and given access to translation queries about $f$, Algorithm \ref{VAC-alg} outputs an array of antichains $A$ such that if $V$ is the set of valid antichains of translations of $x$ which extend $X_{\ell}$, then each antichain in $A$ is extended by an antichain in $V$ and each antichain in $V$ contains a unique antichain in $A$.  Furthermore, $A$ contains the unique antichain which is a valid antichain of all translations of $y$ that extend $X_{\ell}$ for all $y \in L$ such that $y \succeq x$.
\end{Lemma}

\begin{proof}
Since $x$ is $<_{llex}$-least such that $p_x$ is a path to the state at which $p_x$ terminates, $CS_f$ (hence, $\mathcal D$ also) contains $\langle x,X \rangle$.  Furthermore, for each prefix $Y$ such that $X_{\ell} \preceq Y \preceq X$, if $Y$ is not a member of a valid antichain, then $\mathcal D$ contains witnessing strings so that this can be determined using translation queries (this is the content of $N_0(x)$ defined in Section \ref{sec-char-sample}).  Algorithm \ref{VAC-alg} performs exactly those translation queries necessary to determine that $Y$ is not a member of a valid antichain.  Thus, the array of antichains that the algorithm returns will correctly exclude all subsets of maximal antichains that contain such a $Y$.  

We now show that each antichain in the output must be a subset of a valid antichain.  For the sake of a contradiction, suppose an antichain in $A$ contains $Y$ and $Z$ where $X_{\ell} \preceq Y \preceq X$ and $Y$ is a member of a valid antichain, but the unique member of $V$ that contains $Y$ does not contain $Z$.  By the definition of the $N_0$ component of $CS_f$, $\mathcal D$ will contain a string such that when $COMPARE$ is run, $Y$ and $Z$ will be flagged as not members of the same valid antichain.

Let $F$ be the $<_{sdf}$-maximum SDF for $f$.  By the definition of $N_2(f)$, for each $Z \in F(x)$ there must be a $Y$ such that $\langle \hat{x},Y \rangle \in \mathcal D$ and $X_{\ell}Z \preceq Y$.  Consequently, $A$ will contain $\{ X_{\ell} \} * F(x)$, which is the unique antichain which is a valid antichain of the extensions of $X_{\ell}$ in $f(y)$ for all $y \in L$ such that $x \preceq y$.
\end{proof}

\subsubsection{Performing onwarding on a single node}

The next algorithm takes an array of antichains and produces the $<_{ac}$-greatest antichain that appears to be a valid antichain of all trees of outputs on inputs extending~$x$.  As the data may still be incomplete, testing the validity for other trees is done using translation queries.
%
%

\begin{algorithm}[H]\label{TEST-alg}
\caption{Testing an array of antichains against a dataset (TESTVPS)}
\KwData{A string, $x$, over the input alphabet; an array, $A$, of antichains for the output tree of input $\hat{x}$; a collection of translation pairs, $\mathcal D$.}
\KwResult{The $<_{ac}$-greatest member of the array, $A$, for which there is no evidence in $\mathcal D$ that the selected antichain is not valid for all output trees in the future of $x$.}

\For{$i = |A|-1; i \geq 0; i--$}{
	$\mbox{`not valid'} \rightarrow status$\\
	\For{$\langle xy,Z \rangle \in \mathcal D$}{
		\For{$R\in A[i]$}{
			\If{$R \prec Z$}{
				$R^{-1}Z \rightarrow W$\\
				$\mbox{`valid'} \rightarrow status$\\
				\For{$Q\in A[i]$}{
					\If{$[xy,QW]_f = \textsc{false}$}{
						$\mbox{`not valid'} \rightarrow status$\\
						\textbf{break}\\
					}
				}
				\If{$status = \mbox{`not valid'}$}{
					\textbf{break}\\
				}
			}
		}
	}
	\If{$status = \mbox{`valid'}$}{
		\Return{$A[i]$}\\
	}
}

\end{algorithm}

Observe that there will always be a valid antichain that causes the above algorithm to terminate; if there is no other, then it will terminate on $\{ \lambda \}$.  In the following algorithm, we use \textsc{null} to test for the existence of an optional argument.

\begin{algorithm}[H]\label{ONWARD-alg}
\caption{Onwarding a tree-like portion of a transducer (ONWARD)}
\KwData{A string $x$; a transducer, $G$, which is tree-like below a string, $x$; $X_\ell$, the current least translation prefix for $x$; a collection of translation pairs, $\mathcal D$; a set of strings, $S$ (optional).}
\KwResult{A transducer that differs from $G$ only on transitions whose end state is $q_x$ or a state below $x$.}
$S \rightarrow P$\\
\If{$P = \textsc{null}$}{
	$VAC(\mathcal D,x,X_\ell) \rightarrow A$\\
	$TESTVPS(x,A,\mathcal D) \rightarrow P$\\
}

$\troutput(e_x) * P \rightarrow \troutput(e_x)$\\

\For{$y \in \mbox{dom}(\mathcal D) \wedge x \prec y$}{
	$P^{-1} \troutput(e_y) \rightarrow \troutput(e_y)$\\
}

\end{algorithm}

The purpose of Algorithm \ref{ONWARD-alg} is to advance as much translation as possible in a tree-like portion of a transducer.

%
%

\subsubsection{Merging states}

Following conventions presented in \cite{cdlh-book}, we will label states during the learning process as \textsc{red} states if it is not possible to merge them with any $<_{llex}$-lesser state.  Initially, only the input state, $q_{\lambda}$, is a \textsc{red} state.  We proceed through the states in $<_{llex}$-order.  When a new state is found that cannot be merged with any \textsc{red} state, then it becomes a new \textsc{red} state.

The next algorithm we present merges two states if there is no evidence that the underlying transducer behaves differently on extensions of the inputs of the two states.  In this operation, we assume that the first argument is a \textsc{red} state, the second argument is not, and that onwarding has already been performed for both states.  In order to present the algorithm succinctly, we define a function similar to $COMPARE$ from Section \ref{genVacSection}.  Define $FUTURE(x,y,G,\mathcal D) = \textsc{true}$ if 
\begin{align*}
\big(\forall X \in G[p_x] \cap \mbox{ran}(\mathcal D), &Y \in G[p_y] \cap \mbox{ran}(\mathcal D), \langle z,Z \rangle \in \mathcal D\big)\Bigg( \\
&(x \preceq z \wedge X \preceq Z \rightarrow [y(x^{-1}z),Y_0 (X^{-1}Z)]_f = \textsc{true}) \\
&\wedge (y \preceq z \wedge Y \preceq Z \rightarrow [x(y^{-1}z),X_0 (Y^{-1}Z)]_f = \textsc{true}) \Bigg),
\end{align*}
 where $X_0 = \llexleast(G[p_x])$ and $Y_0 = \llexleast(G[p_y])$.  Otherwise, $FUTURE(x,y,G,\mathcal D) = \textsc{false}$.  Note that finding $\llexleast(G[p_x])$ does not require enumeration all elements of $G[p_x]$, which could be exponential in the length of $x$.  To determine $\llexleast(G[p_x])$, one need only find the least element of each set of translations along the path $p_x$.

\begin{algorithm}[H]\label{merge-alg}
\caption{MERGE}
\KwData{A \textsc{red} state, $q_x$; a non-\textsc{red} state, $q_y$; a transducer, $G$, that is tree-like below $q_y$; a collection of translation pairs, $\mathcal D$.}
\KwResult{A transducer; a boolean value of \textsc{true} if the two states have been merged and \textsc{false} otherwise.}

$FUTURE(x,y,G,\mathcal D) \rightarrow status$\\
\If{$status = \textsc{true}$}{
	$q_x \rightarrow end(e_y)$\\
	$\states[G] \setminus\{ q_y \} \rightarrow \states[G]$\\
	\For{$z\in \mbox{dom}(\mathcal D) \wedge z \succ y$}{
		\For{$y \prec w \preceq z$}{
			\If{$q_{x(y^{-1}w)} \in \states[G]$}{
				$\states[G] \setminus\{ q_{w} \} \rightarrow \states[G]$\\
			}
			$q_{x(y^{-1}w)} \rightarrow start(e_{w})$\\
			$q_{x(y^{-1}w) \trinput(e_{w})} \rightarrow end(e_{w})$\\
		}
	}
	\Return{$\langle G, \textsc{true} \rangle$}
}
\Else{
	\Return{$\langle G, \textsc{false} \rangle$}
}

\end{algorithm}

If $G$ is a transducer generated from a dataset, it is likely that $G$ will include non-equivalent states for which there is no evidence in their futures to distinguish them.  Ultimately, this will not be an obstacle to learning because if the characteristic sample has appeared, there will be enough data to distinguish earlier states that will be processed first.

\subsubsection{The learning algorithm}\label{learnAlgSDT}

Our final algorithm combines onwarding and merging into a single process.  We proceed through the states of the initial transducer in $<_{llex}$-order, first onwarding and then attempting to merge with lesser states.  If a state cannot be merged with any lesser state, it is fixed and will not subsequently be changed.  The fact that such states are fixed is recorded by their membership in a set $\textsc{red}$.

%
%

\begin{algorithm}[H]\label{learningAlg}
\caption{Learning an SDT}
\KwData{A collection of translation pairs, $\mathcal D$.}
\KwResult{A transducer.}

$INITIAL(\mathcal D) \rightarrow G_0$\\
$\llexorder(\states[G_0]) \rightarrow S$\\
$q_{\lambda} \rightarrow \textsc{red}[0]$\\
$0 \rightarrow i$\\
\For{$q_x \in S$}{
	\If{$q_x \in \textsc{red} \vee q_{x^-} \not\in \textsc{red}$}{
		\textbf{continue}\\
	}
	\Else{
		$ONWARD(x,G,\llexleast(G[p_x]),\mathcal D) \rightarrow G$\\
		\For{$q_y \in \textsc{red}$}{
			$MERGE(q_y,q_x,G,\mathcal D) \rightarrow \langle G,status \rangle$\\
			\If{$status = \textsc{true}$}{
				\textbf{break}\\
			}
		}
		\If{$status = \textsc{false}$}{
			$x \rightarrow \textsc{red}[i]$\\
			$i++$\\
		}
	}
}

\end{algorithm}

\begin{Lemma}\label{red-paths}
Let $f$ be an SDBL with canonical SDT $G$ and let $\mathcal D$ be a finite set consistent with $f$ which contains $CS_f$.  At every stage during the execution of Algorithm \ref{learningAlg} with input $\mathcal D$ and given access to translation queries about $f$, if $G'$ is the SDT constructed so far and $p_x$ is a path through $G'$ that exclusively involves \textsc{red} states, then $G[p_x] = G'[p_x]$.  Furthermore, if $x$ and $y$ are strings such that such that their unique paths $p_x$ and $p_y$ terminate at different \textsc{red} states of $G'$, then the states $q_x$ and $q_y$ of $G$ are distinct. 
\end{Lemma}

\begin{proof}
Let $G_0$ be the SDT that results from Algorithm \ref{alg-initial} and let $F$ be the $<_{sdf}$-maximum SDF for $f$.  We prove the lemma by induction.  Initially, the only \textsc{red} state is the initial state and the lemma holds trivially.  Now suppose that the lemma holds for $G'$ at the beginning of an iteration of the main for-loop in Algorithm \ref{learningAlg}.  Let $G''$ be the result of executing the next iteration of the for-loop.  

If no new \textsc{red} states have been added, then a previously non-\textsc{red} state, $q_1$, had Algorithm \ref{ONWARD-alg} applied to it and was merged with a \textsc{red} state, $q_0$.  Let $x$ be the $<_{llex}$-least string such that the path $p_x$ in $G'$ ends at $q_x = q_1$.  Since $q_1$ was not a \textsc{red} state, $G'$ must have been tree-like below $x$ and because of the induction hypothesis, the correct least translation prefix for $x$ will have been used by Algorithm \ref{ONWARD-alg}.  Furthermore, since the unique state in $G'$ with a transition to $q_1$ is a \textsc{red} state, by the definition of $N_0(f)$, $\mathcal D$ contains examples to guarantee that Algorithm \ref{VAC-alg} correctly identifies valid antichains.  Thus, Algorithm \ref{ONWARD-alg} must have identified the unique $<_{ac}$-greatest antichain which is a valid antichain of $f(y)$ for all $y \in L$ such that $y \succeq x$.  By the induction hypothesis, $CS_f \subseteq \mathcal D$ must contain examples that uniquely identify the future of $q_0$.  Since precisely those examples from $\mathcal D$ will be tested for $q_1$ using translation queries when Algorithm \ref{merge-alg} is run, the fact that $q_0$ and $q_1$ were merged implies that $q_0 \equiv q_1$.  Consequently, all paths that only visits \textsc{red} states of $G''$ satisfy the lemma.  Since no new \textsc{red} states were introduced, the second conclusion in the statement of the lemma follows automatically from the induction hypothesis.

Now suppose that a new \textsc{red} state, $q_1$, is selected.  Since the state was marked as a \textsc{red} state, no $<_{llex}$-lesser state is equivalent, meaning that if $x$ is $<_{llex}$-least such that $p_x$ in $G'$ ends at $q_1$, then $x$ defines the $<_{llex}$-least path to some state of $G$.  Consequently, $CS_f \subseteq \mathcal D$ contains examples to guarantee that Algorithm \ref{ONWARD-alg} identifies the correct antichain, $F(x)$.  Since $q_1$ was not merged with any existing \textsc{red} state, there must be examples in $\mathcal D$ that distinguish $q_1$ from each of the \textsc{red} states.  Thus, for each such \textsc{red} state, $q$, we know that $q_1 \not\equiv q$.  Since the for-loop in Algorithm \ref{learningAlg} proceeds through the states in $<_{llex}$-order, this means that if $x$ is $<_{llex}$-least such that the path $p_x$ in $G'$ ends at $q_1$, then $x$ is $<_{llex}$-least such that the state $q_x$ in $G$ has the same future as $q_1$.  We may conclude, therefore, that $CS_f \subseteq \mathcal D$ contains examples that uniquely identify the future of $q_1$.  Finally, to prove that $G''[p_z] = G[p_z]$ for any $z$ such that $p_z$ that only visits \textsc{red} states of $G''$ we need only consider paths that end at $q_1$ and do not visit $q_1$ at any other point.  There is only one state, $q$, in $G''$ that has a transition to $q_1$ (because $G''$ is tree-like below $q_1$) and that state is a \textsc{red} state in both $G'$ and $G''$.  Thus, $G[p_{z^-}] = G'[p_{z^-}] = G''[p_{z^-}]$.  The transition from $q$ to $q_1$ with input $z(|z|-1)$ has output $F(x)$, thus $G''[p_z] = G''[p_{z^-}] * F(x) = G[p_z]$.
\end{proof}

\subsection{Learnability of SDBLs} 

\begin{Theorem}\label{learnableThm}
The class of SDBLs is polynomially identifiable in the limit with translation queries.
\end{Theorem}

\begin{proof}
Let $f$ be an SDBL with canonical SDT $G$ and let $\mathcal D$ be a collection of translation pairs consistent with $f$ and containing $CS_f$.  We apply Algorithm \ref{learningAlg} to learn $f$ from $\mathcal D$.  To prove that Algorithm \ref{learningAlg} identifies $f$ in the limit in polynomial time, we must verify three claims.  First, we must show that the size of the chosen characteristic sample is polynomial in the size of the canonical transducer of the target.  Second, we must show that the algorithm terminates within a number of steps that is polynomial in the size of the canonical transducer of the target and in the size of the given data.  Third, we must show the SDT produced by Algorithm \ref{learningAlg} generates $f$.

The first claim is easy.  As noted in the section in which $CS_f$ was defined, $N_0(f), N_1(f)$ and $N_2(f)$ are all polynomial in the size of $G$.

An inspection of the algorithms shows that they converge in polynomial time and that only a polynomial number of translation queries are made.  We conclude that the second claim is true.

Finally, we prove the third claim.  Algorithm \ref{learningAlg} terminates at the point when every state in the SDT generated by Algorithm \ref{alg-initial} has been either marked as a \textsc{red} state or merged with another state.  Suppose that $G'$ is the output of Algorithm \ref{learningAlg}.  When the algorithm terminates, every state is a \textsc{red} state.  Consequntly, for any $x$ for which there is a path $p_x$ in $G'$, by Lemma \ref{red-paths} $G[p_x] = G'[p_x]$.  Thus, we need only prove that for all $x$, if $x$ defines a path through $G$ then $x$ defines a path through $G'$.  By the definition of $N_0(f)$, every transition in $G$ is used at least once by translations in $CS_f$.  Fix $x$ which defines a path through $G$ and let $xa$ be an extension by the single character $a$.  If $xa$ also defines a path through $G$, then $G$ has a transition $e$ such that $start(e)$ is the final state of the path defined by $x$ and $\trinput(e) = a$.  Let $G_0$ be the SDT generated by Algorithm \ref{alg-initial} and let $\langle z_0az_1,Z_0YZ_1 \rangle$ be a translation pair in $CS_f$ such that the path through $G$ defined by $z_0az_1$ uses $e$ when translating the character $a$.  Let $q_{z_0}$ and $q_{z_0a}$ be the states of $G_0$ corresponding to the initial segments $z_0$ and $z_0a$ of $z_0az_1$.  Let $x_0$ be the $<_{llex}$-least string to the final state of the path through $G$ defined by $x$.  Let $q_{x_0}$ be state of $G_0$ corresponding to $x_0$.  By the definition of $N_2(f)$, $CS_f$ contains examples that uniquely identify the future of $q_{x_0}$.  When $FUTURE(x_0,z_0a,G'',\mathcal D)$ is run at some point during the execution of Algorithm \ref{learningAlg} (where $G''$ is the current form of transducer under construction), the two strings will be recognized as having the same futures and will be merged.  Thus, if we assume that the paths through $G'$ defined by $x$ and $x_0$ are the same, then $xa$ defines a path through $G'$.  By induction on the length of $x$, we have shown the every string that defines a path through $G$ also defines a path through $G'$.

%
\end{proof}

\section{Related Results}

We establish the relationship between SDTs and two other classes of transducers that are not entirely deterministic: p-subsequential transducers and transducers that recognize $R_{ffb}$ relations.  We also look at some properties of SDBLs; specifically, we show that SDBLs are not closed under composition or reversal.

\subsection{Other Forms of Non-Determinism}\label{other-non-det-section}
The following definition is adapted from \cite{alla02}.

\begin{Definition}
A transducer is said to be \emph{$p$-subsequential} for some $p \in \mathbb N$ if for every transition $e$ in the transition relation, $|\troutput(e)| = 1$ unless $\trinput(e) = \#$, in which case $|\troutput(e)| \leq p$.  We say the a bi-language is $p$-subsequential if it is generated by a $p$-subsequential transducer.
\end{Definition}

\begin{Definition} \cite{saka08}
A binary relation, $R$, is \emph{finitary} if for every $x \in \mbox{dom}(R)$ the set $\{ y : \langle x,y \rangle \in R \}$ is finite.  If there is a number $n \in \mathbb N$ such that $|\{ y : \langle x,y \rangle \in R \}| \leq n$ for all $x$, then the relation is said to be \emph{bounded}.  Such a relation is \emph{finite-state} if there is a transducer that generates $R$.  We say that a relation is $R_{ffb}$ if it is finitary, finite-state and bounded.  We say that a transducer is $R_{ffb}$ if the bi-language it generates is $R_{ffb}$.
\end{Definition}

\begin{Proposition}
Every $p$-subsequential bi-language is $R_{ffb}$.
\end{Proposition}

\begin{proof}
A $p$-subsequential bi-language is finitary and bounded as every input string has at most $p$ distinct translations.  It is also clearly finite-state as it is generated by a $p$-subsequential transducer.  Thus, every $p$-subsequential bi-language is also $R_{ffb}$.
\end{proof}

\begin{Proposition}\label{psub-not-sdbl}
There is a $p$-subsequential bi-language which is not an SDBL.
\end{Proposition}

\begin{proof}
If $f$ is an SDBL, $x \in \mbox{dom}(f)$ and $X,Y \in f(x)$, then either $X = Y$ or $X$ and $Y$ are incomparable.  Thus, $f: \{a\} \rightarrow \{\{A,AA\}\}$ such that $f(a) = \{A,AA\}$ is not an SDBL.  It is, however, $p$-subsequential.
\end{proof}

\begin{Proposition}
There is an SDBL which is neither $p$-subsequential nor $R_{ffb}$.
\end{Proposition}

\begin{proof}
Let $G$ be a transducer with a single state, which is a terminal state, and one transition which starts and ends at the unique state, has input $a$ and output $\{A,B\}$.   $G$ is shown in Figure \ref{sdt-ffb}.
\begin{figure}[H]
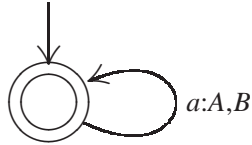

\centering
\resizebox{4cm}{!}{
\xy
(0,0)*++{\xy (0,10)*+{}="init";
(0,0)*+[o]=<20pt>\hbox{}*\frm{oo}="in-st";
(15,-10)*+{}="lp-1-1";
(15,10)*+{}="lp-1-2";
"init";"in-st"**\dir{-} ?>*\dir2{>};
"in-st";"in-st"**\crv{"lp-1-1"&"lp-1-2"} ?>*\dir2{>} ?*_!/-11pt/{_{a:A,B}};
\endxy}
\endxy
}
\caption{An SDT which generates an SDBL which is neither $R_{ffb}$ nor $p$-subsequential.}
\label{sdt-ffb}
\end{figure}
The set of translation pairs recognized by $G$ is $\{ \langle a^n,X \rangle : n \in \mathbb N \wedge X \in \{A,B\}^n \}$, which is not a bounded relation as $a^n$ is in the domain for every $n$ and has $2^n$ translations.  Since it is not bounded, it is neither $R_{ffb}$ nor $p$-subsequential.
\end{proof}

\begin{Proposition}
There is an $R_{ffb}$ bi-language which is neither an SDBL nor $p$-subsequential.
\end{Proposition}

\begin{proof}
Let $G$ be the transducer in the following figure.
\begin{figure}[H]
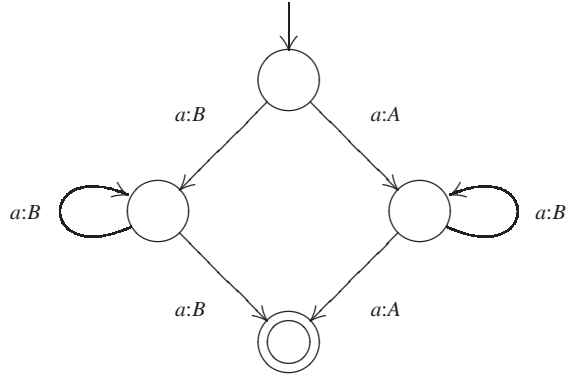

\centering
\resizebox{8cm}{!}{
\xy
(0,0)*++{\xy (0,10)*+{}="init";
(0,0)*+[o]=<20pt>\hbox{}*\frm{o}="in-st";
(15,-15)*+[o]=<20pt>\hbox{}*\frm{o}="ri-st";
(-15,-15)*+[o]=<20pt>\hbox{}*\frm{o}="le-st";
(0,-30)*+[o]=<20pt>\hbox{}*\frm{oo}="fi-st";
(30,-25)*+{}="lp-1-1";
(30,-5)*+{}="lp-1-2";
(-30,-25)*+{}="lp-2-1";
(-30,-5)*+{}="lp-2-2";
"init";"in-st"**\dir{-} ?>*\dir2{>};
"in-st";"le-st"**\dir{-} ?>*\dir2{>} ?*_!/-15pt/{_{a:B}};
"le-st";"fi-st"**\dir{-} ?>*\dir2{>} ?*_!/-15pt/{_{a:B}};
"in-st";"ri-st"**\dir{-} ?>*\dir2{>} ?*_!/15pt/{_{a:A}};
"ri-st";"fi-st"**\dir{-} ?>*\dir2{>} ?*_!/15pt/{_{a:A}};
"ri-st";"ri-st"**\crv{"lp-1-1"&"lp-1-2"} ?>*\dir2{>} ?*_!/-11pt/{_{a:B}};
"le-st";"le-st"**\crv{"lp-2-1"&"lp-2-2"} ?>*\dir2{>} ?*_!/11pt/{_{a:B}};
\endxy}
\endxy
}
\caption{An $R_{ffb}$ bi-language which is not an SDBL.}
\label{diamond-1}
\end{figure}
The set of translation pairs recognized by $G$ is $S = \{ \langle a^{n+2},B^{n+2} \rangle : n \in \mathbb N \} \cup \{ \langle a^{n+2},AB^nA \rangle : n \in \mathbb N \}$.  Suppose $G'$ is an SDT that recognizes $S$.  For every $n$, the input string $a^{n+2}$ has exactly two translations: $B^{n+2}$ and $AB^nA$.  Since $G'$ is has finitely many states, the first transition on the path through $G'$ defined by $a^{n+2}$ with an output other than $\{\lambda\}$ must have $\{AX,BY\}$ as a subset for some strings $X$ and $Y$.  Furthermore, the last transition on the path with output other than $\{\lambda\}$ must have $\{ZA,WB\}$ as a subset for some strings $Z$ and $W$.  This implies that $a^{n+2}$ should have at least four distinct translations, which is a contradiction. 

Similarly, suppose $S$ is recognized by a $p$-subsequential transducer $G_p$.  Since the two translations of $a^{n+2}$ differ at their first character, it must be that for every path through $G_p$ all translation occurs at the final transition (since only terminal transitions may have more than one output).  Given this observation, $G_p$ must have infinitely many states -- one for each string $a^{n+2}$.  This is a contradiction.
\end{proof}

\subsection{Composition and Reversal of SDBLs}

\begin{Definition}
Let $f$ and $g$ be two bi-languages.  We define the \emph{composition of $f$ and $g$} to be the bi-language $h$ such that $\mbox{dom}(h) = \mbox{dom}(f)$ and for every $x \in \mbox{dom}(f)$, $h(x) = \bigcup_{X \in f(x)} g(X)$.  We define the \emph{reversal of $f$} to be the bi-language $k$ with domain $S = \bigcup_{x \in \mbox{dom}(f)}f(x)$ such that $k(X) = \{ x \in \mbox{dom}(f) : X \in f(x) \}$.
\end{Definition}

\begin{Proposition}
There are SDBLs whose composition is not an SDBL.
\end{Proposition}

\begin{proof}
Let $G_f$ and $G_g$ be the following two SDTs and let $f$ and $g$ be the SDBLs generated by them.
\begin{figure}[H]
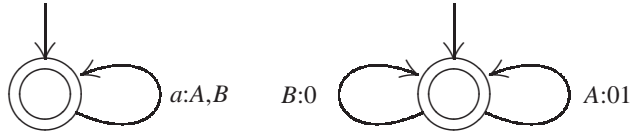

\centering
\resizebox{9cm}{!}{
\xy
(0,0)*++{\xy (0,10)*+{}="init";
(0,0)*+[o]=<20pt>\hbox{}*\frm{oo}="in-st";
(15,-10)*+{}="lp-1-1";
(15,10)*+{}="lp-1-2";
"init";"in-st"**\dir{-} ?>*\dir2{>};
"in-st";"in-st"**\crv{"lp-1-1"&"lp-1-2"} ?>*\dir2{>} ?*_!/-11pt/{_{a:A,B}}; 
\endxy }="G0";
(33,0)*++{\xy (0,10)*+{}="init";
(0,0)*+[o]=<20pt>\hbox{}*\frm{oo}="in-st";
(15,-10)*+{}="lp-1-1";
(15,10)*+{}="lp-1-2";
(-15,-10)*+{}="lp-2-1";
(-15,10)*+{}="lp-2-2";
"init";"in-st"**\dir{-} ?>*\dir2{>};
"in-st";"in-st"**\crv{"lp-1-1"&"lp-1-2"} ?>*\dir2{>} ?*_!/-11pt/{_{A:01}}; 
"in-st";"in-st"**\crv{"lp-2-1"&"lp-2-2"} ?>*\dir2{>} ?*_!/11pt/{_{B:0}}; 
\endxy }="G1";
\endxy
}
\caption{Two SDTs whose composition is is not an SDBL.  On the left, $G_f$; on the right, $G_g$.}
\end{figure}
Observe that if $h$ is the composition of $f$ and $g$, then $h(a) = \{ 0,01 \}$.  Since $0$ is a prefix of $01$, $h$ cannot be an SDBL.
\end{proof}

\begin{Proposition}
There is an SDBL whose reversal is not an SDBL.
\end{Proposition}

\begin{proof}
Let $G$ be the following SDT and let $f$ be the SDBL generated by $G$.  
\begin{figure}[H]
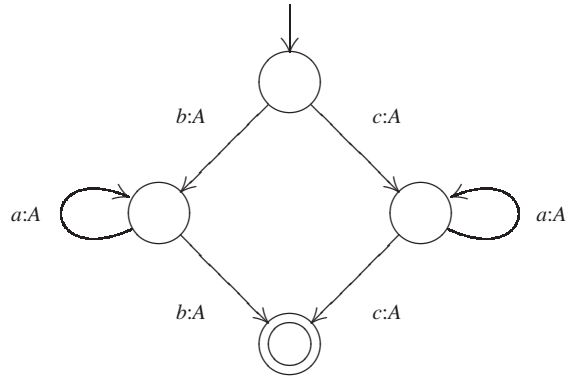

\centering
\resizebox{8cm}{!}{
\xy
(0,0)*++{\xy (0,10)*+{}="init";
(0,0)*+[o]=<20pt>\hbox{}*\frm{o}="in-st";
(15,-15)*+[o]=<20pt>\hbox{}*\frm{o}="ri-st";
(-15,-15)*+[o]=<20pt>\hbox{}*\frm{o}="le-st";
(0,-30)*+[o]=<20pt>\hbox{}*\frm{oo}="fi-st";
(30,-25)*+{}="lp-1-1";
(30,-5)*+{}="lp-1-2";
(-30,-25)*+{}="lp-2-1";
(-30,-5)*+{}="lp-2-2";
"init";"in-st"**\dir{-} ?>*\dir2{>};
"in-st";"le-st"**\dir{-} ?>*\dir2{>} ?*_!/-15pt/{_{b:A}};
"le-st";"fi-st"**\dir{-} ?>*\dir2{>} ?*_!/-15pt/{_{b:A}};
"in-st";"ri-st"**\dir{-} ?>*\dir2{>} ?*_!/15pt/{_{c:A}};
"ri-st";"fi-st"**\dir{-} ?>*\dir2{>} ?*_!/15pt/{_{c:A}};
"ri-st";"ri-st"**\crv{"lp-1-1"&"lp-1-2"} ?>*\dir2{>} ?*_!/-11pt/{_{a:A}};
"le-st";"le-st"**\crv{"lp-2-1"&"lp-2-2"} ?>*\dir2{>} ?*_!/11pt/{_{a:A}};
\endxy}
\endxy
}
\caption{An SDT which generates an SDBL whose reversal is not an SDBL.}
\end{figure}
For each $n$, the reversal of $f$ has two translations of $A^{n+2}$, $ba^nb$ and $ca^nc$.  From this observation we see that the reversal of $f$ cannot be an SDBL for exactly the same reason that the bi-language generated by the transducer in Figure \ref{diamond-1} cannot be an SDBL.
\end{proof}

\section{Conclusion}
We have presented a novel algorithm that learns a powerful class of transducers with the help of reasonable queries.  A probabilistic version of these transducers was defined in \cite{akra13}.  We are unaware of any results involving this version.  As both probabilities and translation queries can serve the purpose of answering questions about translation pairs not present in the given data, it seems possible that probabilistic transducers could be learned without translation queries, with statistical analysis taking the role of translation queries.

The learnability of SDBLs represents a significant advance in our ability to identify underlying structure in the challenging situation where the structure is non-deterministic.  While it does not strictly expand existing models (as can be seen from Proposition \ref{psub-not-sdbl}), it does contribute an enormous class of new bi-languages which are beyond the scope of deterministic transducers.

\section{Acknowledgements}

We would like the thank the anonymous referees for many useful comments, corrections and suggestions -- specifically, for suggesting that we examine the relationship between SDBLs, $p$-subsequential bi-languages and $R_{ffb}$ relations.


\begin{thebibliography}{10}

\bibitem{akra13}
H.~I. Akram.
\newblock {\em Learning Probabilistic Subsequential Transducers}.
\newblock PhD thesis, Technische Universit{\"a}t M{\"u}nchen, 2013.

\bibitem{alla02}
C.~Allauzen and M.~Mohri.
\newblock p-subsequentiable transducers.
\newblock In {\em Implementation and Application of Automata, 7th International
  Conference, \textsc{Ciaa} 2002, Revised Papers}, volume 2608 of {\em
  \textsc{Lncs}}, pages 24--34. Springer-Verlag, 2002.

\bibitem{amen01a}
J.~C. Amengual, J.~M. Bened{\'{\i}}, F.~Casacuberta, A.~Casta{\~n}o,
  A.~Castellanos, V.~M. Jim{\'e}nez, D.~Llorens, A.~Marzal, M.~Pastor, F.~Prat,
  E.~Vidal, and J.~M. Vilar.
\newblock The {EuTrans-I} speech translation system.
\newblock {\em Machine Translation}, 15(1):75--103, 2001.

\bibitem{angl87a}
D.~Angluin.
\newblock Queries and concept learning.
\newblock {\em Machine Learning Journal}, 2:319--342, 1987.

\bibitem{bern06}
M.~Bernard, J.-C. Janodet, and M.~Sebban.
\newblock A discriminative model of stochastic edit distance in the form of a
  conditional transducer.
\newblock In Y.~Sakakibara, S.~Kobayashi, K.~Sato, T.~Nishino, and E.~Tomita,
  editors, {\em Grammatical Inference: Algorithms and Applications, Proceedings
  of \textsc{Icgi} '06}, volume 4201 of {\em \textsc{Lnai}}, pages 240--252.
  Springer-Verlag, 2006.

\bibitem{beros2013}
A.~Beros.
\newblock Anomalous vacillatory learning.
\newblock {\em Journal of Symbolic Logic}, 78(4):1183--1188, 12 2013.

\bibitem{berosND}
A.~Beros.
\newblock Learning theory in the arithmetic hierarchy.
\newblock {\em Journal of Symbolic Logic}, 79(3):908--927, 9 2014.

\bibitem{beros-delahiguera-2014}
Achilles Beros and Colin de~la Higuera.
\newblock A canonical semi-deterministic transducer.
\newblock {\em arXiv preprint arXiv:1405.2476}, 2014.

\bibitem{bers79}
J.~Berstel.
\newblock {\em Transductions and context-free languages}.
\newblock Teubner, Leipzig, 1979.

\bibitem{blum-blum}
M.~Blum and L.~Blum.
\newblock Towards a mathematical theory of inductive inference.
\newblock {\em Information and Control}, 28:125--155, 1975.

\bibitem{carm05}
J.~Carme, R.~Gilleron, A.~Lemay, and J.~Niehren.
\newblock Interactive learning of node selecting tree transducer.
\newblock {\em Machine Learning Journal}, 66:33--67, 2007.

\bibitem{casa99}
F.~Casacuberta and C.~de~la Higuera.
\newblock Optimal linguistic decoding is a difficult computational problem.
\newblock {\em Pattern Recognition Letters}, 20(8):813--821, 1999.

\bibitem{casa00a}
F.~Casacuberta and C.~de~la Higuera.
\newblock Computational complexity of problems on probabilistic grammars and
  transducers.
\newblock In A.~L. de~Oliveira, editor, {\em Grammatical Inference: Algorithms
  and Applications, Proceedings of \textsc{Icgi} '00}, volume 1891 of {\em
  \textsc{Lnai}}, pages 15--24. Springer-Verlag, 2000.

\bibitem{casa04}
F.~Casacuberta and E.~Vidal.
\newblock Machine translation with inferred stochastic finite-state
  transducers.
\newblock {\em Computational Linguistics}, 30(2):205--225, 2004.

\bibitem{clar01b}
A.~Clark.
\newblock Partially supervised learning of morphology with stochastic
  transducers.
\newblock In {\em Proceedings of the Sixth Natural Language Processing Pacific
  Rim Symposium}, pages 341--348, 2001.

\bibitem{clar06a}
A.~Clark.
\newblock Large scale inference of deterministic transductions: Tenjinno
  problem 1.
\newblock In Y.~Sakakibara, S.~Kobayashi, K.~Sato, T.~Nishino, and E.~Tomita,
  editors, {\em Grammatical Inference: Algorithms and Applications, Proceedings
  of \textsc{Icgi} '06}, volume 4201 of {\em \textsc{Lnai}}, pages 227--239.
  Springer-Verlag, 2006.

\bibitem{cost04}
F.~Coste, D.~Fredouille, C.~Kermorvant, and C.~de~la Higuera.
\newblock Introducing domain and typing bias in automata inference.
\newblock In G.~Paliouras and Y.~Sakakibara, editors, {\em Grammatical
  Inference: Algorithms and Applications, Proceedings of \textsc{Icgi} '04},
  volume 3264 of {\em \textsc{Lnai}}, pages 115--126. Springer-Verlag, 2004.

\bibitem{higu97}
C.~de~la Higuera.
\newblock Characteristic sets for polynomial grammatical inference.
\newblock {\em Machine Learning Journal}, 27:125--138, 1997.

\bibitem{cdlh-book}
Colin de~la Higuera.
\newblock {\em Grammatical Inference: Learning Automata and Grammars}.
\newblock Cambridge University Press, 2010.

\bibitem{gold67}
E.~M. Gold.
\newblock Language identification in the limit.
\newblock {\em Information and Control}, 10(5):447--474, 1967.

\bibitem{jech2003set}
Thomas Jech.
\newblock {\em Set theory}.
\newblock Springer Monographs in Mathematics. Springer-Verlag, Berlin, 2003.
\newblock The third millennium edition, revised and expanded.

\bibitem{kerm02a}
C.~Kermorvant and C.~de~la Higuera.
\newblock Learning languages with help.
\newblock In P.~Adriaans, H.~Fernau, and M.~van Zaannen, editors, {\em
  Grammatical Inference: Algorithms and Applications, Proceedings of
  \textsc{Icgi} '02}, volume 2484 of {\em \textsc{Lnai}}, pages 161--173.
  Springer-Verlag, 2002.

\bibitem{kunen80}
K.~Kunen.
\newblock {\em Set theory: an introduction to independence proofs}, volume 102
  of {\em Studies in Logic and the Foundations of Mathematics}.
\newblock North-Holland Publishing Co., Amsterdam-New York, 1980.

\bibitem{mohr97}
M.~Mohri.
\newblock Finite-state transducers in language and speech processing.
\newblock {\em Computational Linguistics}, 23(3):269--311, 1997.

\bibitem{mohr00a}
M.~Mohri.
\newblock Minimization algorithms for sequential transducers.
\newblock {\em Theoretical Computer Science}, 234:177--201, 2000.

\bibitem{mohr00b}
M.~Mohri, F.~C.~N. Pereira, and M.~Riley.
\newblock The design principles of a weighted finite-state transducer library.
\newblock {\em Theoretical Computer Science}, 231(1):17--32, 2000.

\bibitem{nerode}
A.~Nerode.
\newblock Linear automaton transformations.
\newblock {\em Proceedings of the American Mathematical Society},
  9(4):541--544, 1958.

\bibitem{nivat68}
Maurice Nivat.
\newblock Transductions des langages de chomsky.
\newblock In {\em Annales de l'institut Fourier}, volume 18/1, pages 339--455.
  Institut Fourier, 1968.

\bibitem{onci93}
J.~Oncina, P.~Garc{\'{\i}}a, and E.~Vidal.
\newblock Learning subsequential transducers for pattern recognition
  interpretation tasks.
\newblock {\em Pattern Analysis and Machine Intelligence}, 15(5):448--458,
  1993.

\bibitem{onci96}
J.~Oncina and M.~A. Var{\'o}.
\newblock Using domain information during the learning of a subsequential
  transducer.
\newblock In L.~Miclet and C.~de~la Higuera, editors, {\em Proceedings of
  \textsc{Icgi} '96}, number 1147 in \textsc{Lnai}, pages 301--312.
  Springer-Verlag, 1996.

\bibitem{roar07}
B.~Roark and R.~Sproat.
\newblock {\em Computational Approaches to Syntax and Morphology}.
\newblock Oxford University Press, 2007.

\bibitem{saka09}
J.~Sakarovitch.
\newblock {\em Elements of Automata Theory}.
\newblock Cambridge University Press, 2009.

\bibitem{saka08}
Jacques Sakarovitch and Rodrigo De~Souza.
\newblock On the decidability of bounded valuedness for transducers.
\newblock In {\em Mathematical Foundations of Computer Science 2008}, volume
  5162 of {\em \textsc{Lncs}}, pages 588--600. Springer, 2008.

\bibitem{star06}
B.~Starkie, M.~van Zaanen, and D.~Estival.
\newblock The {T}enjinno machine translation competition.
\newblock In Y.~Sakakibara, S.~Kobayashi, K.~Sato, T.~Nishino, and E.~Tomita,
  editors, {\em Grammatical Inference: Algorithms and Applications, Proceedings
  of \textsc{Icgi} '06}, volume 4201 of {\em \textsc{Lnai}}, pages 214--226.
  Springer-Verlag, 2006.

\bibitem{vali84}
L.~G. Valiant.
\newblock A theory of the learnable.
\newblock {\em Communications of the Association for Computing Machinery},
  27(11):1134--1142, 1984.

\bibitem{vila96}
J.~M. Vilar.
\newblock Query learning of subsequential transducers.
\newblock In L.~Miclet and C.~de~la Higuera, editors, {\em Proceedings of
  \textsc{Icgi} '96}, number 1147 in \textsc{Lnai}, pages 72--83.
  Springer-Verlag, 1996.

\bibitem{vila00}
J.~M. Vilar.
\newblock Improve the learning of subsequential transducers by using alignments
  and dictionaries.
\newblock In A.~L. de~Oliveira, editor, {\em Grammatical Inference: Algorithms
  and Applications, Proceedings of \textsc{Icgi} '00}, volume 1891 of {\em
  \textsc{Lnai}}, pages 298--312. Springer-Verlag, 2000.

\bibitem{cast93}
J.M. Vilar, V.~M. Jim{\'{e}}nez, J{-}C. Amengual, A.~Castellanos, D.~Llorens,
  and E.~Vidal.
\newblock Text and speech translation by means of subsequential transducers.
\newblock {\em Natural Language Engineering}, 2(4):351--354, 1996.

\end{thebibliography}
\end{document}